\def\year{2020}\relax
\documentclass[letterpaper]{article} 
\usepackage{aaai20}  
\usepackage{times}  
\usepackage{helvet} 
\usepackage{courier}  
\usepackage[hyphens]{url}  
\usepackage{graphicx} 

\usepackage{booktabs}
\usepackage{multirow}
\usepackage{amsthm}
\usepackage{amssymb}
\usepackage{amsmath}

\usepackage{graphicx}  
\newtheorem{principle}{Principle}
\newtheorem{lemma}{Lemma}
\usepackage{algorithm}
\usepackage{algorithmic}
\newcommand{\citet}[1]{\citeauthor{#1}~\shortcite{#1}}
\newcommand{\citep}{\cite}

\title{Improved Knowledge Distillation via Teacher Assistant}
\newcommand*\samethanks[1][\value{footnote}]{\footnotemark[#1]}

\author{Seyed Iman Mirzadeh\thanks{Equal Contribution} \textsuperscript{\rm 1},
Mehrdad Farajtabar\samethanks  \textsuperscript{\rm 2},
Ang Li \textsuperscript{\rm 2},\\ \Large \textbf{Nir Levine \textsuperscript{\rm 2}, Akihiro Matsukawa\thanks{Work Done at DeepMind} \textsuperscript{\rm 3}, Hassan Ghasemzadeh \textsuperscript{\rm 1}}\\
\textsuperscript{\rm 1} Washington State University, WA, USA  \\
\textsuperscript{\rm 2} DeepMind, CA, USA  \\
\textsuperscript{\rm 3} D.E. Shaw, NY, USA\\
\textsuperscript{\rm 1} \{seyediman.mirzadeh,hassan.ghasemzadeh\}@wsu.edu\\
\textsuperscript{\rm 2} \{farajtabar,anglili,nirlevine\}@google.com\\
\textsuperscript{\rm 3} akihiro.matsukawa@gmail.com\\
}

\begin{document}
\maketitle

\begin{abstract}
Despite the fact that deep neural networks are powerful models and achieve appealing results on many tasks, they are too large to be deployed on edge devices like smartphones or embedded sensor nodes. 
There have been efforts to compress these networks, and a popular method is knowledge distillation, where a large (teacher) pre-trained network is used to train a smaller (student) network. However, in this paper, we show that the student network performance degrades when the gap between student and teacher is large. Given a fixed student network, one cannot employ an arbitrarily large teacher, or in other words, a teacher can effectively transfer its knowledge to students up to a certain size, not smaller. 
To alleviate this shortcoming, we introduce multi-step knowledge distillation, which employs an intermediate-sized network (teacher assistant) to bridge the gap between the student and the teacher. 
Moreover, we study the effect of teacher assistant size and extend the framework to multi-step distillation.
Theoretical analysis and extensive experiments on CIFAR-10,100 and ImageNet datasets and on CNN and ResNet architectures substantiate the effectiveness of our proposed approach.
\end{abstract}

\section{Introduction}
Deep neural networks have achieved state of the art results in a variety of applications such as computer vision~\citep{Huang2017DenselyCC,Hu2018SqueezeandExcitationN}, speech recognition~\citep{Han2017TheC2} and natural language processing~\citep{Devlin2018BERTPO}. Although it is established that introducing more layers and more parameters often improves the accuracy of a model, big models are computationally too expensive to be deployed on devices with limited computation power such as mobile phones and embedded sensors.
Model compression techniques have emerged to address such issues, \textit{e.g.}, parameter pruning and sharing~\citep{han2015compression}, low-rank factorization~\citep{Tai2015ConvolutionalNN} and knowledge distillation~\citep{bucilua2006model,kd}. Among these approaches, knowledge distillation has proven a promising way to obtain a small model that retains the accuracy of a large one. It works by adding a term to the usual classification loss that encourages the student to mimic the teacher's behavior.

However, we argue that knowledge distillation is not always effective, especially when the gap (in size) between teacher and student is large. To illustrate, we ran experiments that show surprisingly a student model distilled from a teacher with more parameters(and better accuracy) performs worse than the same one distilled from a smaller teacher with a smaller capacity. Such scenarios seem to impact the efficacy of knowledge distillation where one is given a small student network and a pre-trained large one as a teacher, both fixed and (wrongly) presumed to form a perfect transfer pair.

\begin{figure}[t]
\centering
\includegraphics[width=0.75\linewidth]{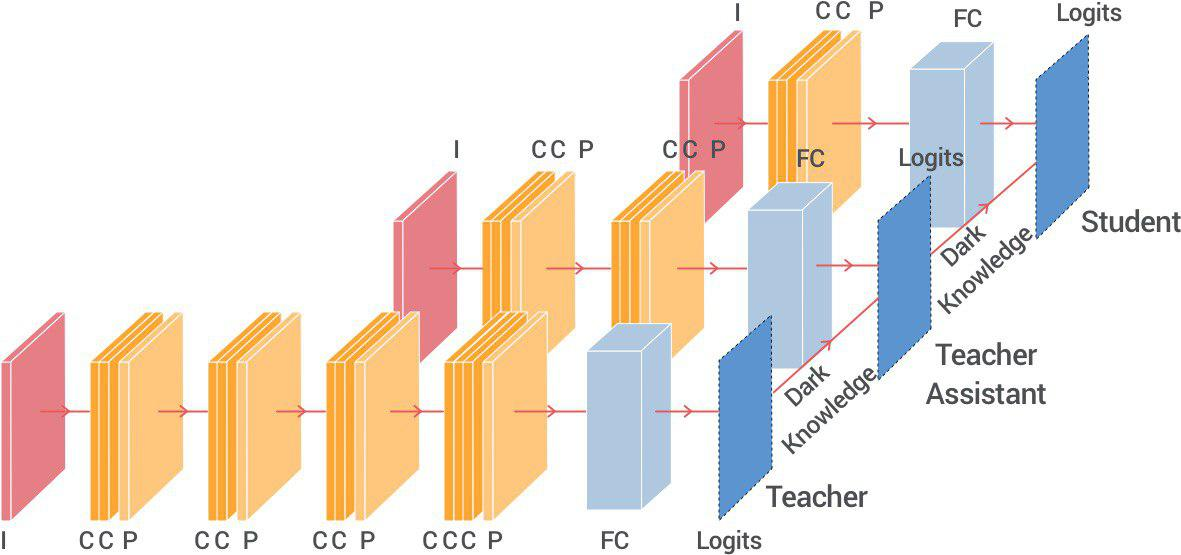}
\caption{TA fills the gap between student $\&$ teacher}
\label{fig:intro} 
\end{figure}

Inspired by this observation, we propose a new distillation framework called Teacher Assistant Knowledge Distillation (TAKD), which introduces intermediate models as teacher assistants (TAs) between the teacher and the student to fill in their gap (Figure~\ref{fig:intro}). TA models are distilled from the teacher, and the student is then only distilled from the TAs.

Our contributions are: (1) We show that the size (capacity) gap between teacher and student is important. To the best of our knowledge, we are the first to study this gap and verify that the distillation performance is not at its top with the largest teacher; (2) We propose a teacher assistant based knowledge distillation approach to improve the accuracy of student network in the case of extreme compression; (3) We extend this framework to include a chain of multiple TAs from teacher to student to further improve the knowledge transfer and provided some insights to find the best one; (4) Through extensive empirical evaluations and a theoretical justification, we show that introducing intermediary TA networks improves the distillation performance.

\section{Related Work} \label{sec:relatedwork}
We discuss in this section related literature in  knowledge distillation and neural network compression.

\textbf{Model Compression.} Since our goal is to train a small, yet accurate network, this work is related to model compression. There has been an interesting line of research that compresses a large network by reducing the connections based on weight magnitudes \citep{han2015compression,Li2016PruningFF} or importance scores \citep{Yu_2018_CVPR}.
The reduced network is fine-tuned on the same dataset to retain its accuracy. Another line of research focuses on distilling the original (large) network to a smaller network \citep{polino2018model,wang-aaai19}, in which case the smaller network is more flexible in its architecture design does not have to be a sub-graph of the original network.

\textbf{Knowledge Distillation.} Originally proposed by~\citet{bucilua2006model} and popularized by~\citet{kd}  knowledge distillation compress the knowledge of a large and computational expensive model (often an ensemble of neural networks) to a single computational efficient neural network. The idea of knowledge distillation is to train the small model, the student, on a transfer set with soft targets provided by the large model, the teacher. Since then, knowledge distillation has been widely adopted in a variety of learning tasks \citep{kd-gift,Yu_2017_ICCV,kickstart-rl,chen2017learning}. Adversarial methods also have been utilized for modeling knowledge transfer between teacher and student~\citep{heo2018improving,xu2018training,wang2018kdgan,Wang2018AdversarialLO}.

There have been works studying variants of model distillation that involve multiple networks learning at the same time. \citet{romero2014fitnets} proposed to transfer the knowledge using not only the logit layer but earlier ones too. To cope with the difference in width, they suggested a regressor to connect teacher and student's intermediate layers. Unfortunately, there is not a principled way to do this. To solve this issue, \citet{kd-gift,Yu_2017_ICCV} used a shared representation of layers, however, it's not straightforward to choose the appropriate layer to be matched.
\citet{czarnecki2017sobolev} minimized the difference between teacher and student derivatives of the loss combined with the divergence from teacher predictions while \citet{Tarvainen2017MeanTA} uses averaging model weights instead of target predictions.
\citet{urban2016deep} trained a network consisting of an ensemble of 16 convolutional neural networks and compresses the learned function into shallow multilayer perceptrons.
To improve the student performance, \citet{sau2016deep} injected noise into teacher logits to make the student more robust. Utilizing multiple teachers were always a way to increase robustness.
\citet{deep-mutual} proposed deep mutual learning which allows an ensemble of student models to learn collaboratively and teach each other during training. KL divergences between pairs of students are added into the loss function to enforce the knowledge transfer among peers. 
\citet{you2017learning}~proposed a voting strategy to unify multiple relative dissimilarity information provided by multiple teacher networks.
\citet{codistillation} introduced an efficient distributed online distillation framework called co-distillation and argue that distillation can even work when the teacher and student are made by the same network architecture. The idea is to train multiple models in parallel and use distillation loss when they are not converged, in which case the model training is faster and the model quality is also improved.

However, the effectiveness of distilling a large model to a small model has not yet been well studied. Our work differs from existing approaches in that we study how to improve the student performance given fixed student and teacher network sizes, and introduces intermediate networks with a moderate capacity to improve distillation performance.
Moreover, our work can be seen as a complement that can be combined with them and improve their performance.

\textbf{Distillation Theory.} Despite its huge popularity, there are few systematic and theoretical studies on how and why knowledge distillation improves neural network training.
The so-called \emph{dark knowledge} transferred in the process helps the student learn the finer structure of teacher network.

\citet{kd} argues that the success of knowledge distillation is attributed to the logit distribution of the incorrect outputs, which provides information on the similarity between output categories. \citet{furlanello2018born} investigated the success of knowledge distillation via gradients of the loss where the soft-target part acts as an importance sampling weight based on the teachers confidence in its maximum value. 
\citet{deep-mutual} analyzed knowledge distillation from the posterior entropy viewpoint claiming that soft-targets bring robustness by regularizing a much more informed choice of alternatives than blind entropy regularization. Last but not least, \citet{lopez2015unifying} studied the effectiveness of knowledge distillation from the perspective of learning theory~\citep{vapnik1998statistical} by studying the estimation error in empirical risk minimization framework.

In this paper, we take this last approach to support our claim on the effectiveness of introducing an intermediate network between student and teacher. Moreover, we empirically analyze it via visualizing the loss function.

\section{Assistant based Knowledge Distillation} \label{sec:method}

\subsection{Background and Notations}

The idea behind knowledge distillation is to have the student network (S) be trained not only via the information provided by true labels but also by observing how the teacher network (T) represents and works with the data. The teacher network is sometimes deeper and wider~\citep{kd}, of similar size~\citep{codistillation,deep-mutual}, or shallower but wider~\citep{romero2014fitnets}.

Let $a_t$ and $a_s$ be the logits (the inputs to the final softmax) of the teacher and student network, respectively.
In classic supervised learning, the mismatch between output of student network softmax$(a_s)$ and the ground-truth label $y_r$ is usually penalized using cross-entropy loss 
\begin{equation}
\mathcal{L}_{SL} =  \mathcal{H}(\text{softmax}(a_s), y_r).
\end{equation}
In knowledge distillation, originally proposed by~\citet{bucilua2006model,ba2014deep} and popularized by~\citet{kd}, one also tries to match the softened outputs of student $y_s = $ softmax($a_s/\tau)$ and teacher $y_t = $softmax($a_t/\tau)$  via a KL-divergence loss
\begin{equation}
\mathcal{L}_{KD} = \tau^2 KL(y_s, y_t)
\end{equation}
Hyperparameter
$\tau$ referred to temperature is introduced to put additional control on softening of signal arising from the output of the teacher network. The student network is then trained under the following loss function:
\begin{equation}
    \mathcal{L}_{\text{student}} = (1-\lambda) \mathcal{L}_{SL} + \lambda \mathcal{L}_{KD}
\end{equation}
where $\lambda$ is a second hyperparameter controlling the trade-off between the two losses. 
We refer to this approach as Baseline Knowledge Distillation (BLKD) through the paper.

\begin{figure}[t!]
\centering
\resizebox{0.95\linewidth}{!}{

\begin{tabular}{cc}
   \hspace{-3mm}
   \includegraphics[width=0.5\linewidth]{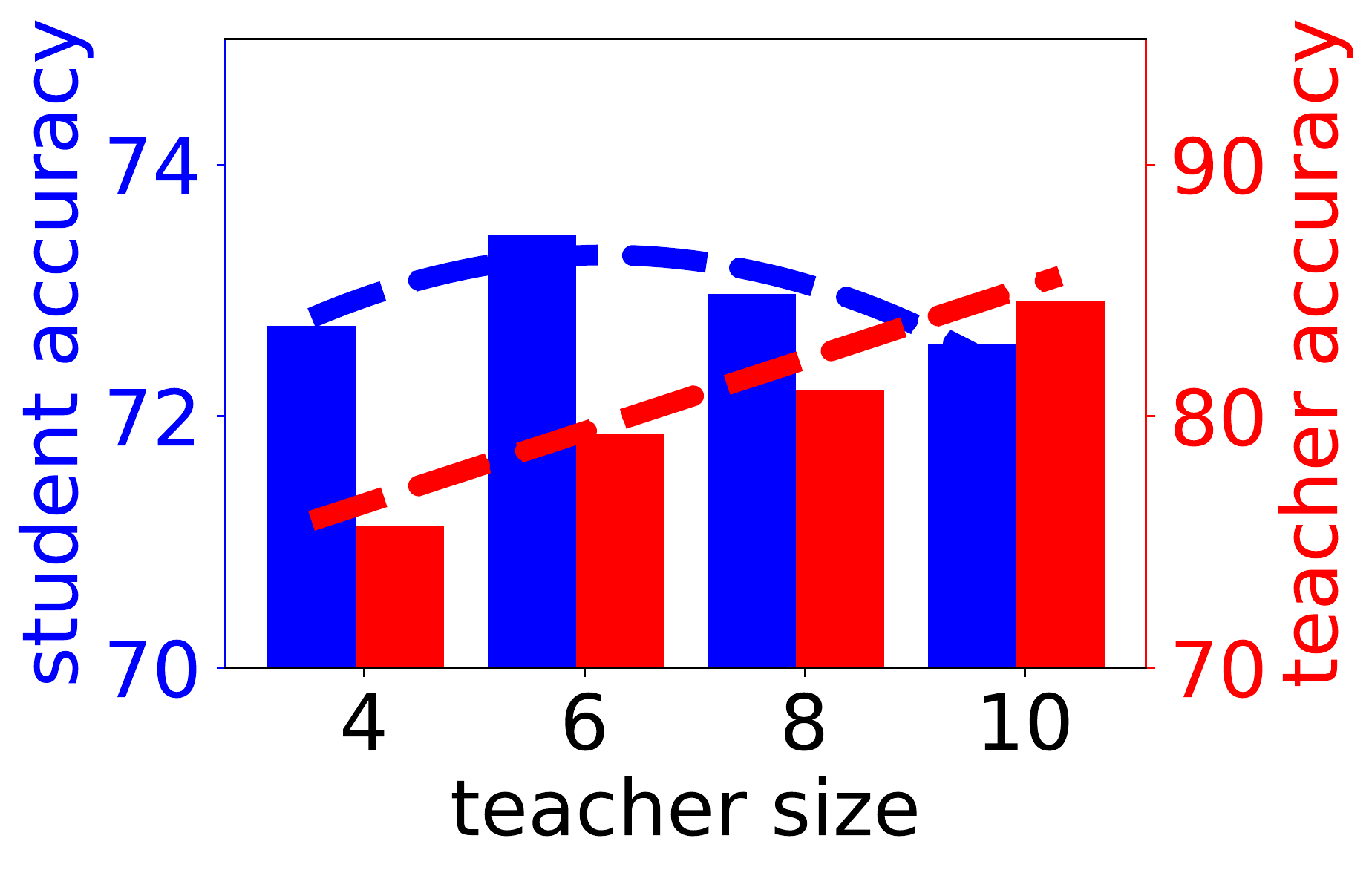}  & \hspace{-4mm}  \includegraphics[width=0.5\linewidth]{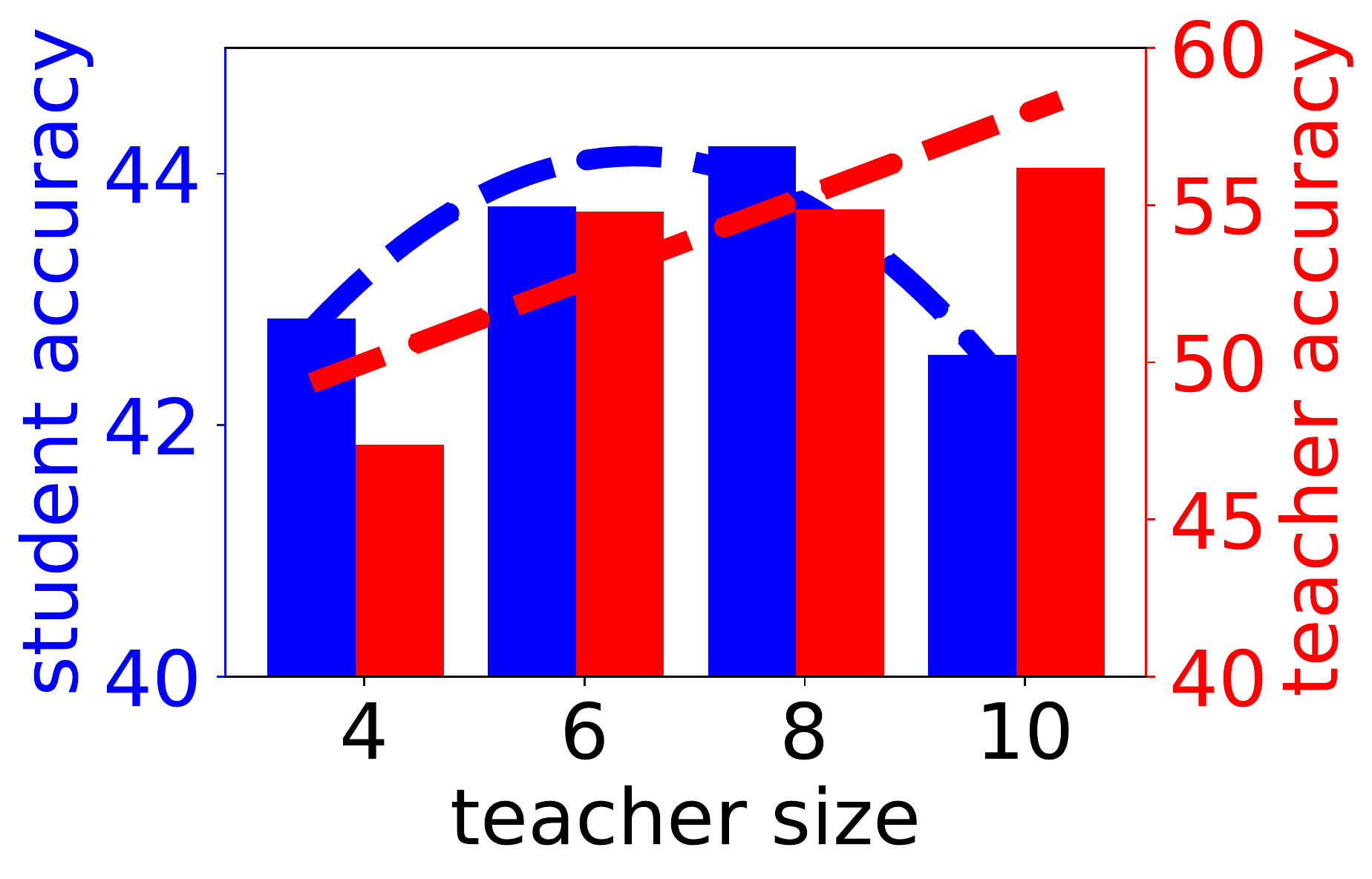} \\
   a) CIFAR-10  & b) CIFAR-100
\end{tabular}
}

\caption{Distillation performance with increasing teacher size. The number of convolutional layers in student is 2.}
\label{fig:increaseing-teacher}
\end{figure}

\subsection{The Gap Between Student and Teacher}
Given a fixed student network, \emph{e.g.}, a Convolutional Neural Network (CNN) with 2 layers to be deployed on a small embedded device, and a pool of larger pre-trained CNNs, which one should be selected as the teacher in the knowledge distillation framework?
The first answer is to pick the strongest which is the biggest one. However, this is not what we observed empirically as showing in Figure~\ref{fig:increaseing-teacher}. Here, a plain CNN student with $2$ convolutional layers is being trained via distillation with similar but larger teachers of size $4$, $6$, $8$, and $10$ on both CIFAR-10 and CIFAR-100 datasets. By size, we mean the number of convolutional layers in the CNN. This number is roughly proportional to the actual size or number of parameters of the neural network and proxy its capacity. Note that they are usually followed by max-pooling or fully connected layers too. We defer the full details on experimental setup to experiments section.

With increasing teacher size, its own (test) accuracy increases (plotted in red on the right axis). However, the trained student accuracy first increases and then decreases (depicted in blue on the left axis).
To explain this phenomenon, we can name a few factors that are competing against each other when enlarging the teacher:
\begin{enumerate}
\item Teacher's performance increases, thus it provides better supervision for the student by being a better predictor.
\item The teacher is becoming so complex that the student does not have the sufficient capacity or mechanics to mimic her behavior despite receiving hints.
\item Teacher's certainty about data increases, thus making its logits (soft targets) less soft. This weakens the knowledge transfer which is done via matching the soft targets. 
\end{enumerate}

Factor 1 is in favor of increasing the distillation performance while factors 2 and 3 are against it. Initially, as the teacher size increases, factor 1 prevails; as it grows larger, factors 2 and 3 dominate.

Similarly, imagine the dual problem. We are given a large teacher network to be used for training smaller students, and we are interested in knowing for what student size this teacher is most beneficial in the sense of boosting the accuracy against the same student learned from scratch. As expected and illustrated in Figure~\ref{fig:increaseing-student}, by decreasing student size, factor 1 causes an increase in the student's performance boost while gradually factors 2 and 3 prevail and worsen the performance gain.

\begin{figure}[t!]
\centering
\begin{tabular}{cc}
   \includegraphics[width=0.425\linewidth]{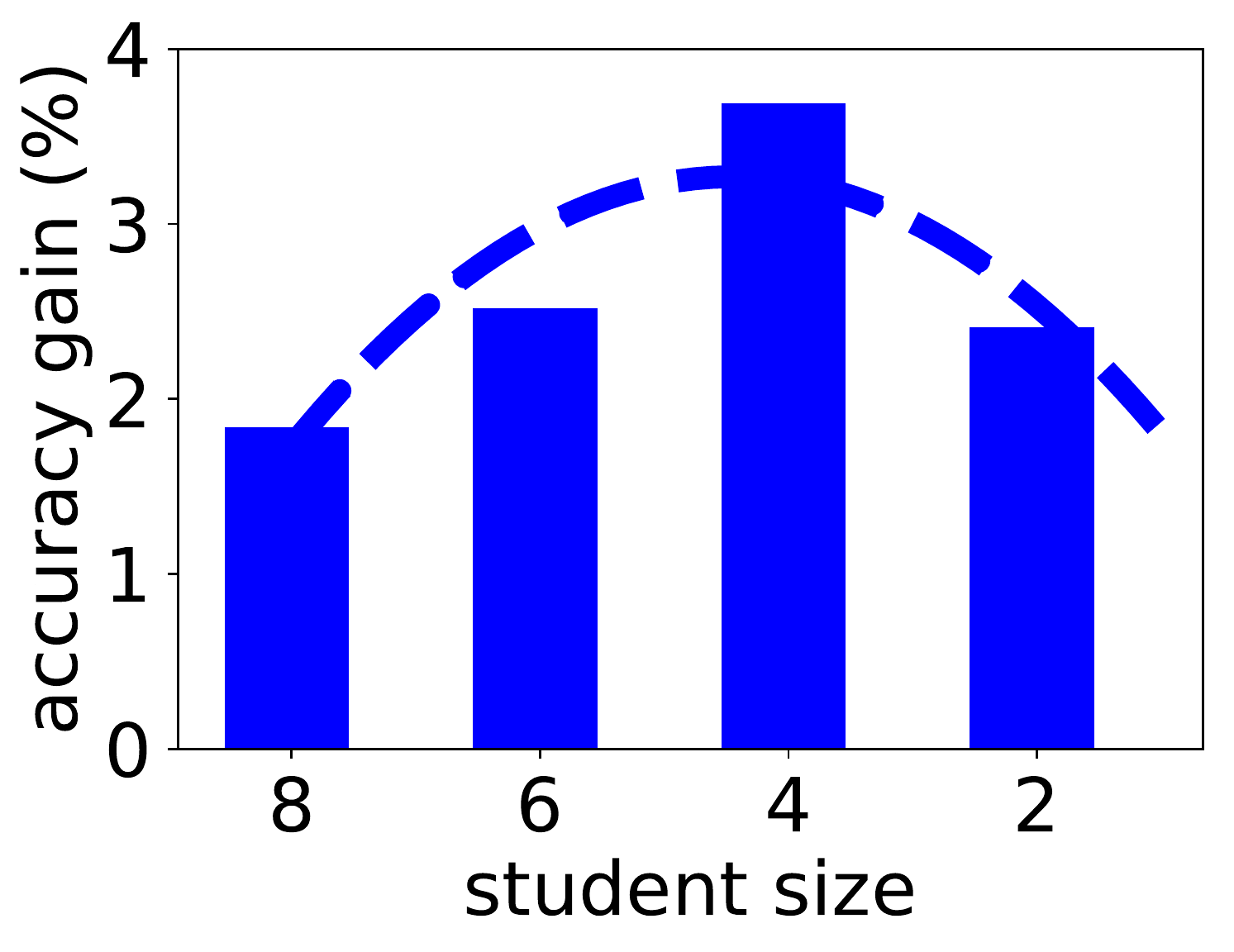}  &   \includegraphics[width=0.425\linewidth]{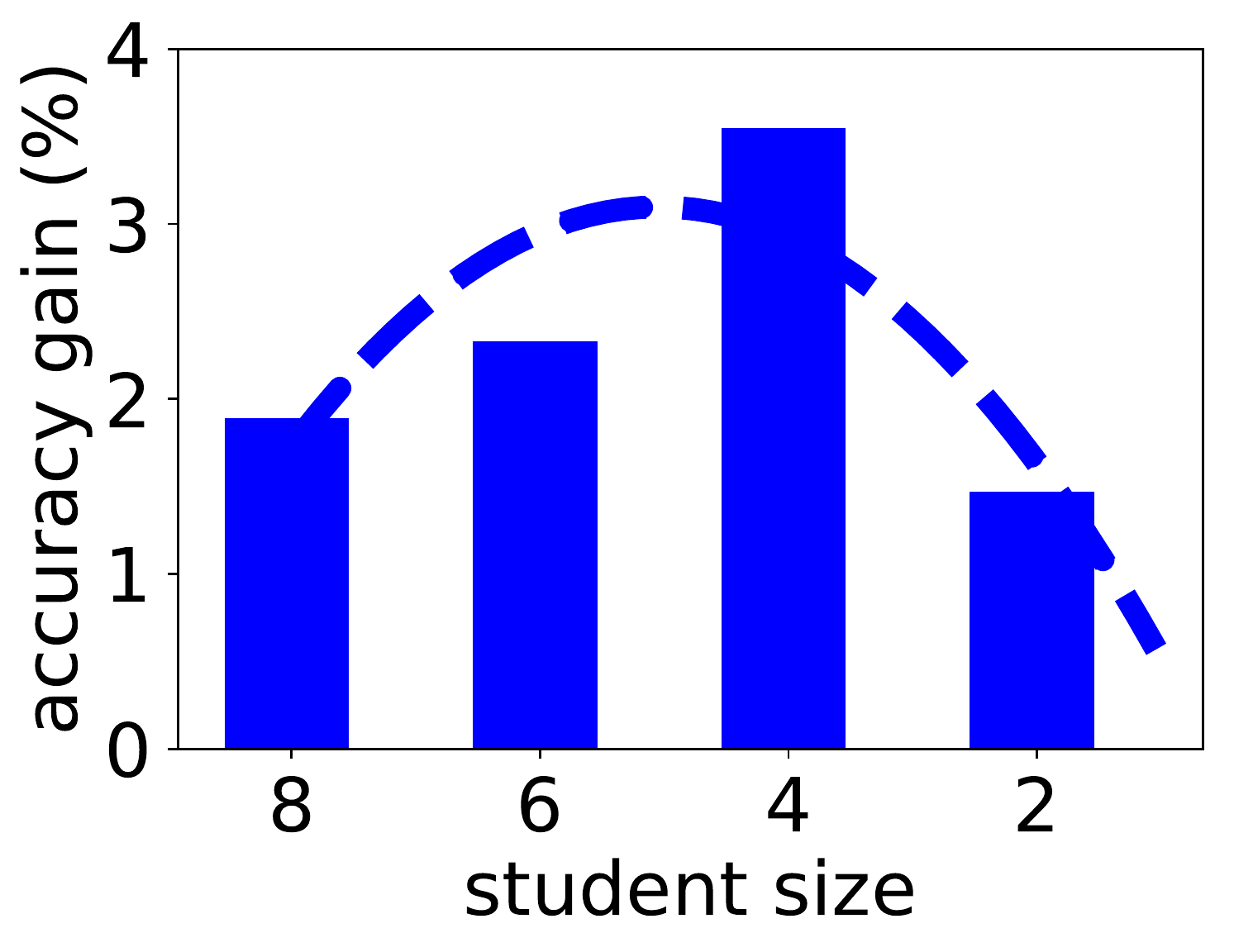}  \\
   a) CIFAR-10  &  b) CIFAR-100
\end{tabular}
\caption{Percentage of distilled student performance increase over the performance when it learns from scratch with varying student size. The teacher has 10 layers.}
\label{fig:increaseing-student} 
\end{figure}

\subsection{Teacher Assistant Knowledge Distillation (TAKD)}
Imagine a real-world scenario where a pre-trained large network is given, and we are asked to distill its knowledge to a fixed and very small student network. 
The gap discussed in the previous subsection makes the knowledge distillation less effective than it could be. Note that we cannot select the teacher size or the student size to maximize the performance. Both are fixed and given.

In this paper, we propose to use intermediate-size networks to fill in the gap between them. The teacher assistant (TA) lies somewhere in between teacher and student in terms of size or capacity. First, the TA network is distilled from the teacher. Then, the TA plays the role of a teacher and trains the student via distillation. This strategy will alleviate factor 2 in the previous subsection by being closer to the student than the teacher. Therefore, the student is able to fit TA's logit distribution more effectively than that of the teacher's. It also alleviates factor 3 by allowing softer (and maybe) less confident targets.
In terms of factor 1, a TA may degrade the performance, however, as we will see in experiments and theoretical analysis sections, both empirical results and theoretical analyses substantiate the effectiveness (improved performance) of TAKD. This happens because encouraging positively correlated factors (like 2 and 3) outweighs the performance loss due to negative ones (like 1).

It will be demonstrated in experiments that TA with any intermediate size always improves the knowledge distillation performance. However, one might ask what the optimal TA size for the highest performance gain is? If one TA improves the distillation result, why not also train this TA via another distilled TA?
Or would a TA trained from scratch be as effective as our approach? In the following sections, we try to study and answer these questions from empirical perspectives complemented with some theoretical intuitions.


\section{Experimental Setup}
\label{sec:setup}
We describe in this section the settings of our experiments.

\textbf{Datasets.}
We perform a set of experiments on two standard datasets CIFAR-10 and CIFAR-100 and one experiment on the large-scale ImageNet dataset.
The datasets consist  of $32 \times 32$ RGB images. The task for all of them is to classify images into image categories. CIFAR-10, CIFAR-100 and ImageNet contain 10 and 100 and 1000 classes, respectively. 

\begin{table}[t!]
\centering
\caption{Comparison on evaluation accuracy between our method (TAKD) and baselines. For CIFAR, plain (S=$2$, TA=$4$, T=$10$) and for ResNet  (S=$8$, TA=$20$, T=$110$) are used. For ImageNet, ResNet (S=$14$, TA=$20$, T=$50$) is used. Higher numbers are better.}
\label{table:general-comprisson}
\vskip 3pt
\resizebox{0.95\linewidth}{!}{
\begin{tabular}{ccccc}
\toprule
\bf Model & \bf Dataset & \bf NOKD & \bf BLKD & \bf TAKD \\ \midrule
\multirow{2}{*}{CNN} & CIFAR-10 & 70.16 & 72.57 & 73.51 \\ 
& CIFAR-100 & 41.09 & 44.57  & 44.92 \\ 
\midrule
\multirow{2}{*}{ResNet}& CIFAR-10 & 88.52 & 88.65  &  88.98 \\ 
& CIFAR-100 & 61.37 & 61.41  & 61.82 \\ 
\midrule
ResNet & ImageNet & 65.20 & 66.60 & 67.36\\
\bottomrule
\end{tabular}
}
\end{table}

\textbf{Implementation.}   We used PyTorch \citep{paszke2017automatic} framework for the implementation\footnote{Codes and Appendix are available at the following address: \url{https://github.com/imirzadeh/Teacher-Assistant-Knowledge-Distillation}} and as a preprocessing step, we transformed images to ones with zero mean and standard deviation of 0.5. For optimization, we used stochastic gradient descent with Nesterov momentum of 0.9 and learning rate of 0.1 for 150 epochs. For experiments on plain CNN networks, we used the same learning rate, while for ResNet training we decrease learning rate to 0.01 on epoch 80 and 0.001 on epoch 120. We also used weight decay with the value of $0.0001$ for training ResNets. To attain reliable results, we performed all the experiments with a hyper-parameter optimization toolkit \citep{MSNNI} which uses a tree-structured Parzen estimator to tune hyper-parameters as explained in \citep{Bergstra2011AlgorithmsFH}. Hyper-parameters include distillation trade-off $\lambda$ and temperature $\tau$  explained in the previous section. 
It's notable that all the accuracy results reported in this paper, are the top-1 test accuracy reached by the hyper-parameter optimizer after running each experiment for 120 trials.

\textbf{Network Architectures.}
We evaluate the performance of the proposed method on two architectures. The first one is a VGG like architecture (plain CNN) consists of convolutional cells (usually followed by max pooling and/or  batch normalization) ended with fully connected layer. We take the number of convolutional cells as a proxy for size or capacity of the network.  The full details of each plain CNN network is provided in the appendix\footnote{Appendix is available along with the code repository.}.
We also used ResNet as a more advanced CNN architecture with skip connections. We used the structures proposed in the original paper~\citep{he2016deep}. The number of blocks in the ResNet architecture is served as a proxy for the size or flexibility of the network.

\begin{table}[t!]
\caption{Student's accuracy given varied TA sizes for (S=$2$, T=$10$)}
\vskip 3pt
\label{table:varying-ta-size}
\centering
\resizebox{0.90\linewidth}{!}{

\begin{tabular}{ccccc}
\toprule
\bf Model & \bf Dataset & \bf TA=8 & \bf TA=6 & \bf TA=4 \\ \midrule
\multirow{2}{*}{CNN} & CIFAR-10 & 72.75 & 73.15 & 73.51 \\ 
& CIFAR-100 & 44.28 & 44.57  &  44.92 \\
\bottomrule
\end{tabular}
}
\end{table}

\begin{table}[t!]
\caption{Student's accuracy given varied TA sizes for (S=$8$, T=$110$)}
\vskip 3pt
\centering
\label{table:varying-ta-size-resnet}
\resizebox{0.95\linewidth}{!}{
\begin{tabular}{cccccc}
\toprule
\bf Model & \bf Dataset & \bf TA=56 &  \bf TA=32 & \bf TA=20 & \bf TA=14 \\ \midrule
\multirow{2}{*}{ResNet } & CIFAR-10 & 88.70 & 88.73 & 88.90 & 88.98 \\ 
& CIFAR-100 & 61.47 & 61.55  & 61.82 & 61.5  \\ 
\bottomrule
\end{tabular}
}
\end{table}


\section{Results and Analysis} 
\label{sec:exp}
In this section, we evaluate our proposed Teacher Assistant Knowledge Distillation  (TAKD) and investigate several important questions related to this approach. Throughout this section, we use S=$i$ to represent the student network of size $i$, T=$j$ to represent a teacher network of size $j$ and TA=$k$ to represent a teacher assistant network of size $k$.
As a reminder by size we mean the number of convolutional layers for plain CNN and ResNet blocks for the case of ResNet. These serve as a proxy for the size or the number of parameters or capacity of the network.

\subsection{Will TA Improve Knowledge Distillation?}
First of all, we compare the performance of our Teacher Assistant based method (TAKD) with the baseline knowledge distillation (BLKD) and with training normally without any distillation (NOKD) for the three datasets and two architectures.
Table~\ref{table:general-comprisson} shows the results.
It is seen the proposed  method outperforms both the baseline knowledge distillation and the normal training of neural networks by a reasonable margin. We include ImageNet dataset only for this experiment to demonstrate TAKD works for the web-scale data too. For the rest of the paper we work with CIFAR10 and CIFAR100.

\begin{figure}[t!]
\centering
\resizebox{0.85\linewidth}{!}{

\begin{tabular}{cc}t
    \hspace{-3mm}
    \includegraphics[width=0.499\linewidth]{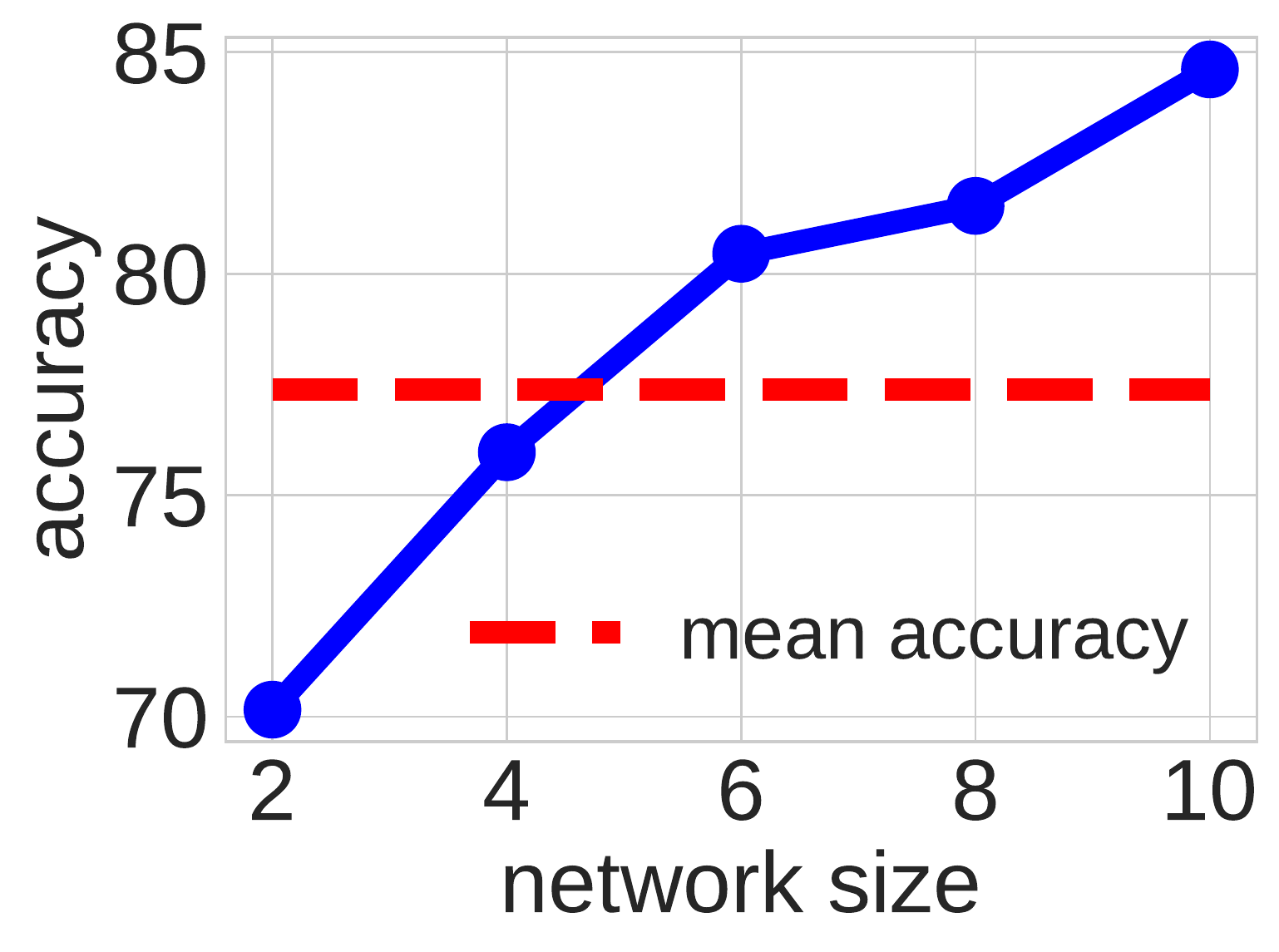}  & 
    \hspace{-3mm}
    \includegraphics[width=0.499\linewidth]{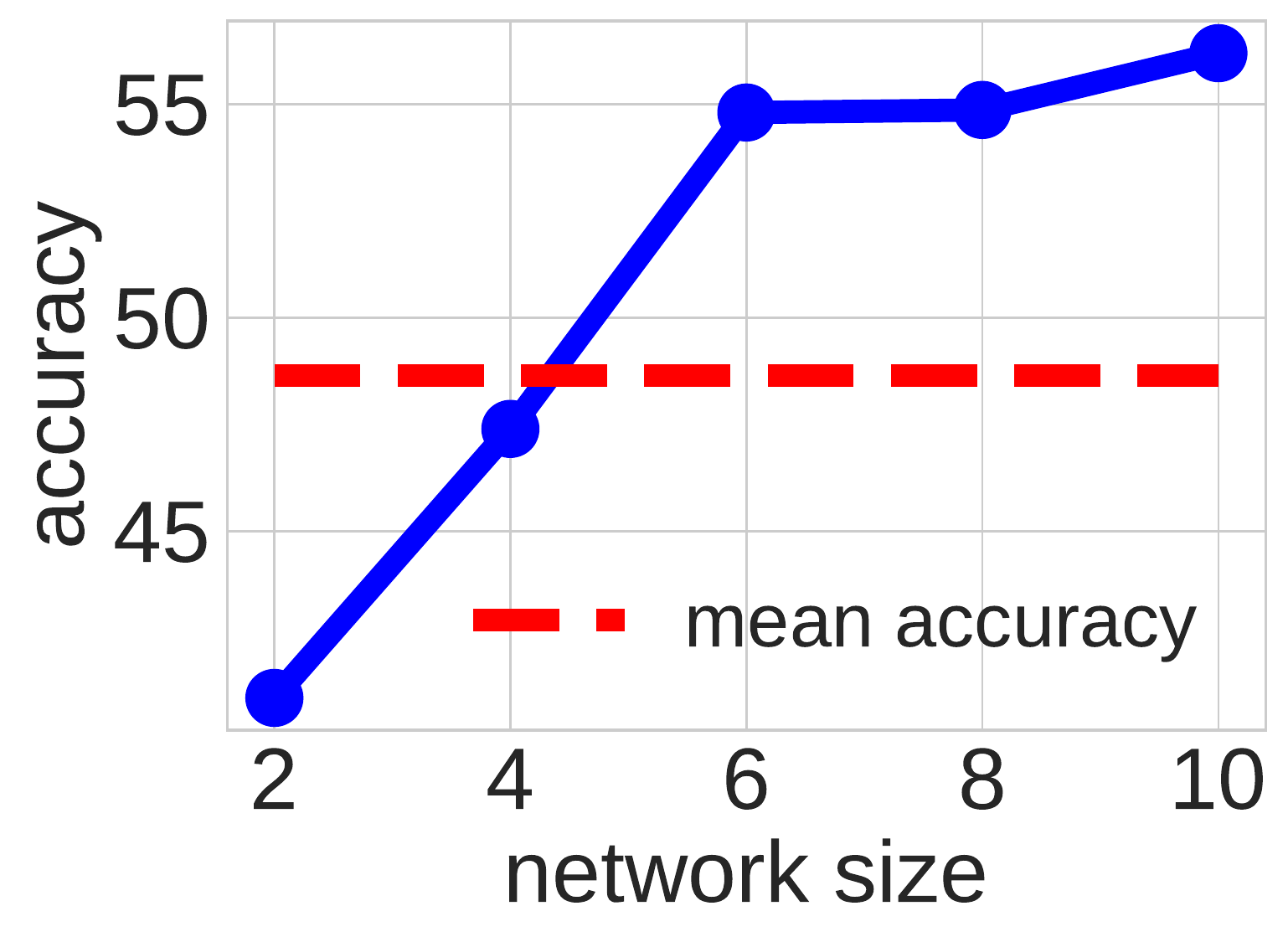}  \\
    (a) CIFAR-10, Plain CNN & (b) CIFAR-100, Plain CNN \\
     \hspace{-3mm}
     \includegraphics[width=0.499\linewidth]{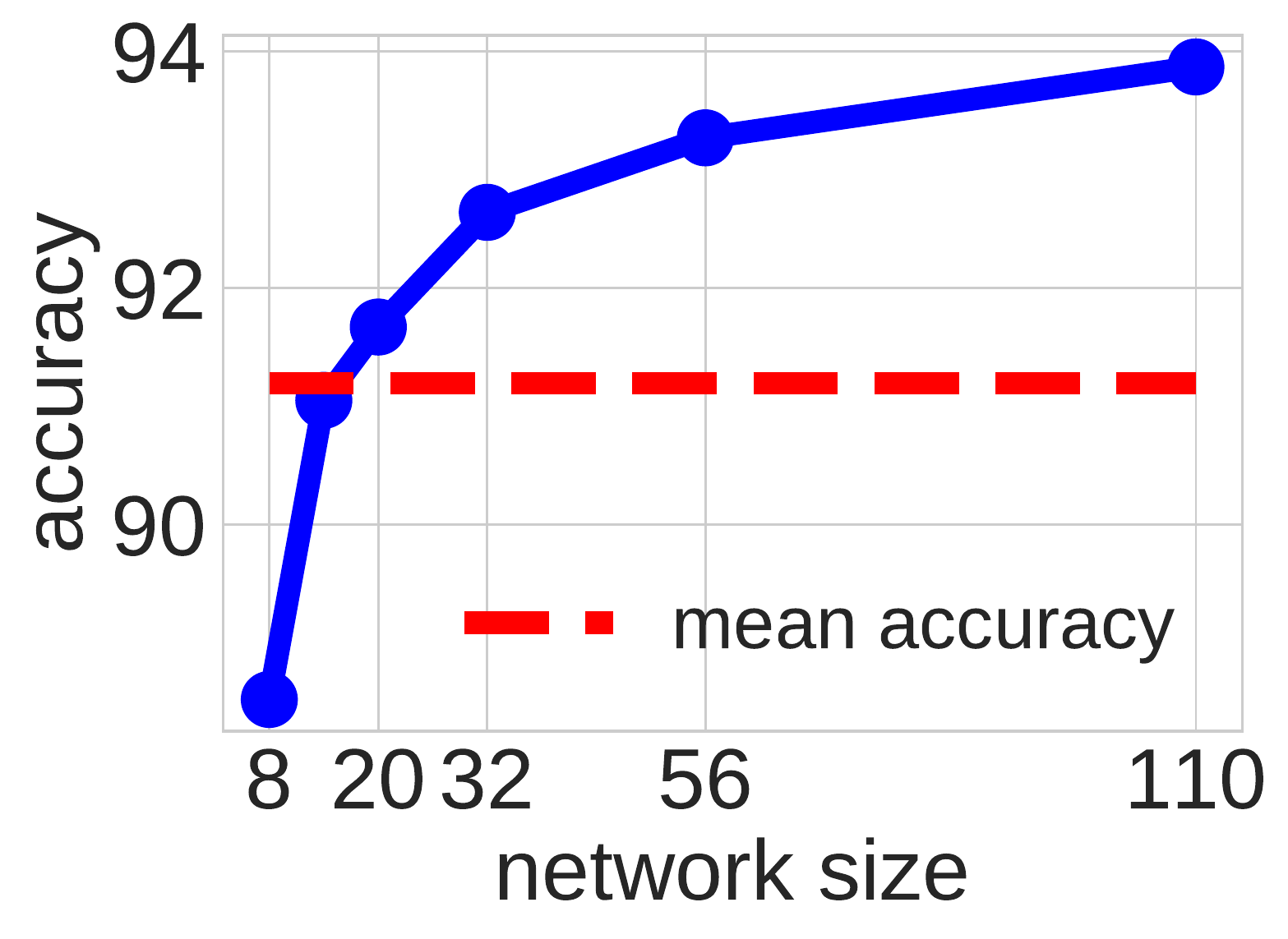} &  
     \hspace{-3mm}
     \includegraphics[width=0.499\linewidth]{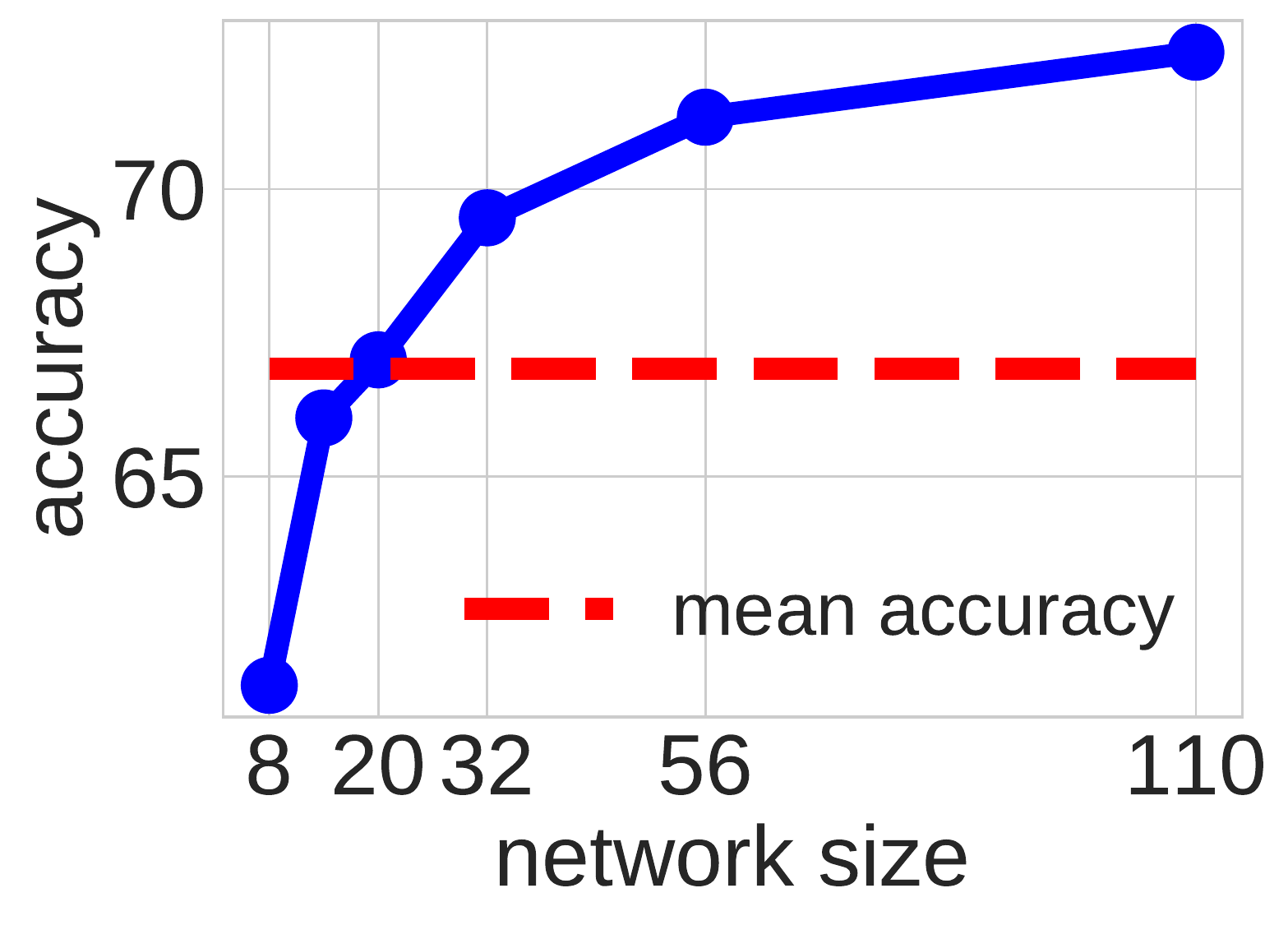} \\
     (c) CIFAR-10, ResNet & (d) CIFAR-100, ResNet
\end{tabular}
}
\caption{Accuracy of training from scratch for different network sizes.  The dashed red line shows the average performance of the teacher and student.}
\label{fig:self-eprformance}
\end{figure}

\subsection{What is the Best TA Size?}
The benefits of having a teacher assistant as an intermediary network for transferring knowledge comes with an essential burden -- selecting the proper TA size.  We evaluate the student's accuracy given varied TA sizes for plain CNN in Table~\ref{table:varying-ta-size} and for ResNet in~\ref{table:varying-ta-size-resnet}, respectively. 

The first observation is that having a TA (of any size) improves the result compared to BLKD and NOKD reported in Table~\ref{table:general-comprisson}. 
Another observation is that for the case of CNN, TA=$4$ performs better than TA=$6$ or TA=$8$. One might naturally ask why $4$ is the best while $6$ seems to be better bridge as it is exactly lies between $2$ and $10$?
Alternatively, we note that for both CIFAR-10 and CIFAR-100, the optimal TA size ($4$) is actually placed close to the middle in terms of average accuracy rather than the average of size.  Figure~\ref{fig:self-eprformance}-a,b depicts the accuracy of a trained neural network with no distillation in blue while the mean accuracy between S=$2$ and T=$10$ is depicted in red dashed line. The figure shows that for both of them, size $4$ is closer to the mean value compared to $6$ or $8$. 
For ResNet in Table~\ref{table:varying-ta-size-resnet} for CIFAR-10, TA=$14$ is the optimum, while, for CIFAR-100, TA=$20$ is the best. Interestingly, Figure~\ref{fig:self-eprformance}-c,d confirms that for CIFAR-10, TA=$14$ is closer to the mean performance of size $8$ and $110$ while TA=$20$ is so for CIFAR-100.
Incorporating a TA with size close to the average performance of teacher and student seems to be a reasonable heuristic to find the optimal TA size, however, more systematic theoretical and empirical investigation remains an interesting venue for future work.

\begin{figure*}[t]
\centering
\begin{tabular}{c}
 \includegraphics[width=0.75\linewidth]{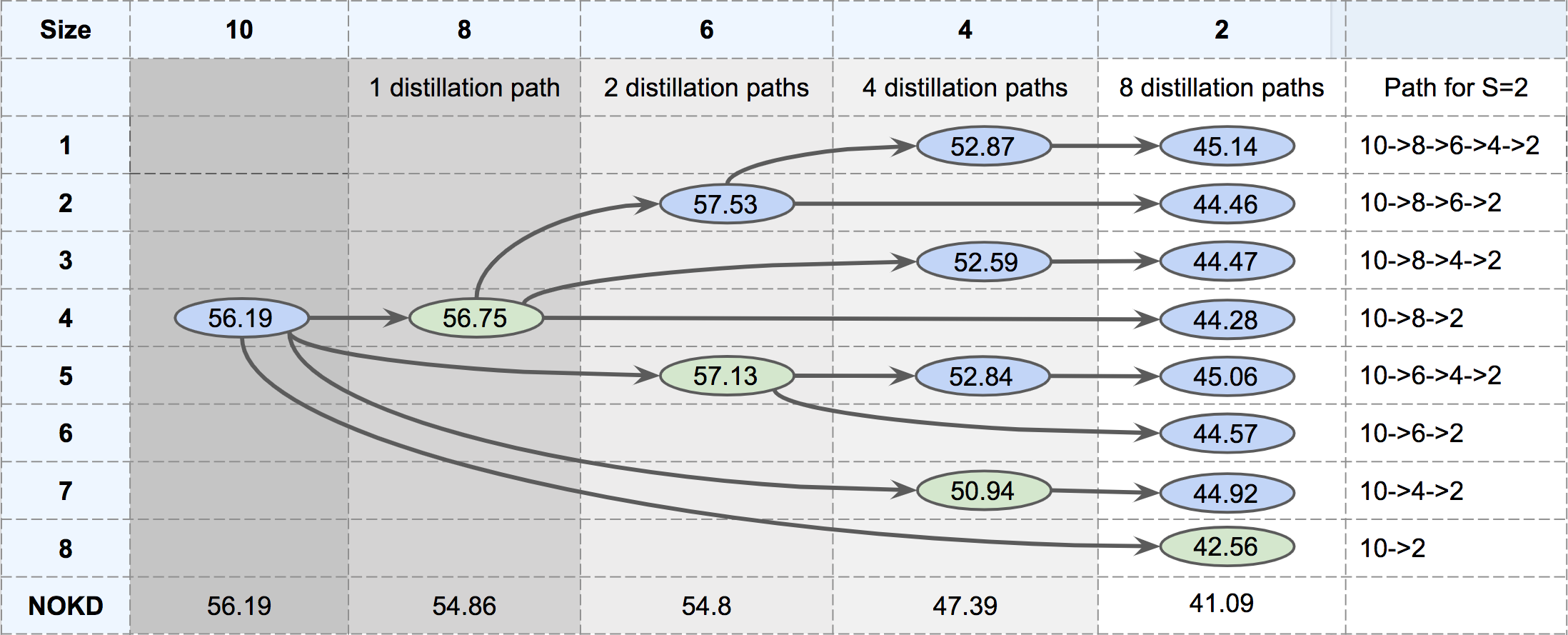}
\end{tabular}
\caption{Distillation paths for plain CNN on CIFAR-100 with T=$10$ }
\label{fig:multi-ta}
\end{figure*}


\subsection{Why Limiting to 1-step TA?}
\label{subsec:multi-ta}
We have seen that for CNN networks on CIFAR-100, incorporating a TA=$4$ between S=$2$ and T=$10$ improves the student. 
However, to train TA=$4$ via distillation from T=$10$, one may propose to put another TA (Say TA =$ 6$) in between to enhance the TA training via another distillation. 
 Using a simplified notation we represent the above sequential distillation process by the \emph{distillation path} $10\rightarrow 6 \rightarrow 4\rightarrow 2$.
Even, one could go further and do a distillation via the path $10\rightarrow 8\rightarrow 6\rightarrow 4 \rightarrow 2$.

To investigate this extension we evaluate all the possible distillation paths and show their outcomes in a single graph in Figure~\ref{fig:multi-ta}. To simplify the presentation we only include networks with even numbers of layers. The numbers in each oval are the accuracy on CIFAR-100 trained on CNN network using the corresponding distillation paths.
A benefit of this visualization is not only that we can study the transfer results to S=$2$, but also for intermediate sizes.
Given $n$ possible intermediate networks, there are $2^n$ possible paths. For example, the 4 possible paths to transfer from T=10 to S=4 are shown in column associated to size 4. For better comparison, the direct transfer (associated to BLKD) are colored in green while the performance without distillation (NOKD) is shown in the last row.

By studying this figure we get interesting insights. Firstly, it is clear that, for all the student sizes (S=$2$,$4$,$6$), TAKD works better than BLKD or NOKD. No matter how many TAs are included in the distillation path, one can obtain better students compared to BLKD and NOKD.
Secondly, the column associated with size 2 reveals that all multi-step TAKD variants work comparably good and considerably better than BLKD and NOKD. 
Thirdly, for S=$2$ and S=$4$, a full path going through all possible intermediate TA networks performs the best.
According to these observations, one can choose a distillation path based on the time and computing resources available. Without any constraint, a full distillation path is optimal (refer to appendix for details).
However, an interesting extension is to limit the number of intermediate TA networks. Can we find the best path in that setting? Given a student and teacher is there a way to automatically find the best path given the constraints? In the appendix section we provide a discussion for these problems.

\subsection{Comparison with Other Distillation Methods}
Since the rediscovery of the basic knowledge distillation method~\citep{kd} many variants of it has been proposed. In 
Fig~\ref{fig:gaps}-right
we have compared the performance of our proposed framework via a single TA with some of the most recent state-of-the-art ones reported and evaluated by~\citet{heo2018improving}.
The 'FITNET'~\citep{romero2014fitnets} proposed to match the knowledge in the intermediate layers.
The method denoted as 
'AT' proposed spatial transfer between teacher and student~\citep{zagoruyko2016paying}. 
In the 'FSP' method, a channel-wise correlation matrix is used as the medium of knowledge transfer~\citep{kd-gift}.
The method 'BSS'~\citep{heo2018improving} trains a student classifier based on the adversarial samples supporting the decision boundary. For these the numbers are reported from the paper~\citep{heo2018improving}. To make a fair comparison, we used exactly the same setting for CIFAR-10 experiments.
In addition to 50K-10K training-test division, all classifiers were trained $80$ epochs. Although we found that more epochs (e.g. $160$) further improves our result, we followed their setting for a fair comparison. ResNet26 is the teacher and ResNet8 and ResNet14 are the students. In addition, we compared with deep mutual learning, 'MUTUAL', with our own implementation of the proposed algorithm in ~\citep{deep-mutual} where the second network is the teacher network. Also, since deep mutual learning needs an initial training phase for both networks, we did this initialization phase for $40$ epochs for both networks and then, trained both networks mutually for $80$ epochs, equal to other modes. For our method, we used TAs ResNet20 and ResNet14 for students ResNet14 and ResNet8, respectively. It's seen that our TA-trained student outperforms all of them.
Note that our proposed framework can be combined with all these variants to improve them too.


\section{Why Does Distillation with TA work?}
\label{sec:why}
In this section, we try to shed some light on why and how our TA based knowledge distillation  is better than the baselines.

\subsection{Theoretical Analysis}
According to the VC theory~\citep{vapnik1998statistical} one can decompose the classification error of a classifier $f_s$ as
\begin{equation}
    R(f_s) - R(f_r) \le O \left(
    \frac{|\mathcal{F}_s|_C}{n^{\alpha_{sr}}} \right) + \epsilon_{sr},
    \label{eq:student-from-real}
\end{equation}
where, the $O(\cdot)$ and  $\epsilon_{sr}$ terms are the estimation and approximation error, respectively. 
The former is related to the statistical procedure for learning given the number of data points, while the latter is characterized by the capacity of the learning machine.
Here, $f_r \in \mathcal{F}_r$ is the real (ground truth) target function and $f_s \in \mathcal{F}_s$ is the student function, $R$ is the error, $|\cdot|_C$ is some function class capacity measure, $n$ is the number of data point, and finally $\frac{1}{2} \le  \alpha_{sr} \le 1$ is related to the learning rate acquiring small values close to $\frac{1}{2}$ for difficult problems while being close to $1$ for easier problems. Note that $\epsilon_{sr}$ is the approximation error of the student function class $\mathcal{F}_s$ with respect to $f_r \in \mathcal{F}_r$. Building on the top  of~\citet{lopez2015unifying}, we extend their result and investigate why and when introducing a TA improves knowledge distillation.
In Equation~\eqref{eq:student-from-real} student learns from scratch (NOKD). 
Let $f_t \in \mathcal{F}_t$ be the teacher function, then 
\begin{equation}
    R(f_t) - R(f_r) \le O \left(
    \frac{|\mathcal{F}_t|_C}{n^{\alpha_{tr}}} \right) + \epsilon_{tr},
    \label{eq:teacher-from-real}
\end{equation}
where, $\alpha_{tr}$ and $\epsilon_{tr}$ are correspondingly defined for teacher learning from scratch.  
Then, we can transfer the knowledge of the teacher directly to the student and retrieve the baseline knowledge distillation (BLKD). To simplify the argument we assume the training  is done via pure distillation ($\lambda = 1$):
\begin{equation}
    R(f_s) - R(f_t) \le O \left(
    \frac{|\mathcal{F}_s|_C}{n^{\alpha_{st}}} \right) + \epsilon_{st},
    \label{eq:student-from-teacher}
\end{equation}
where $\alpha_{st}$ and $ \epsilon_{st}$ are associated to student learning from teacher.
If we combine Equations~\eqref{eq:teacher-from-real} and \eqref{eq:student-from-teacher} we get
\begin{equation}
    O \left(
    \frac{|\mathcal{F}_t|_C}{n^{\alpha_{tr}}}
    + \frac{|\mathcal{F}_s|_C}{n^{\alpha_{st}}}
    \right) + \epsilon_{tr} + \epsilon_{st}  \le O \left(
    \frac{|\mathcal{F}_s|_C}{n^{\alpha_{sr}}} \right) + \epsilon_{sr}
    \label{eq:blkd}
\end{equation}
to hold for BLKD to be effective.
In line with our finding, but with a little different formulation,~\citet{lopez2015unifying} pointed out $|\mathcal{F}_t|_C$ should be small, otherwise the BLKD would not outperform NOKD. We acknowledge that similar to ~\citet{lopez2015unifying}, we work with the upper bounds not the actual performance and also in an asymptotic regime. 
Here we built on top of their result and put a (teacher) assistant between student and teacher
\begin{equation}
    R(f_s) - R(f_a) \le O \left(
    \frac{|\mathcal{F}_s|_C}{n^{\alpha_{sa}}} \right) + \epsilon_{sa},
    \label{eq:student-from-ta}
\end{equation}
and, then the TA itself learns from the teacher
\begin{equation}
    R(f_a) - R(f_t) \le O \left(
    \frac{|\mathcal{F}_a|_C}{n^{\alpha_{at}}} \right) + \epsilon_{at},
    \label{eq:ta-from-teacher}
\end{equation}
where, $\alpha_{sa}$, $\epsilon_{sa}$, 
$\alpha_{at}$, and $\epsilon_{at}$ are defined accordingly.
Combing Equations~\eqref{eq:teacher-from-real}, ~\eqref{eq:student-from-ta},  and~\eqref{eq:ta-from-teacher}  leads to the following equation that needs to be satisfied in order to TAKD outperforms BLKD and NOKD, respectively:

\begin{figure*}[t]
\centering
\begin{tabular}{cc}
\includegraphics[width=0.20\linewidth]{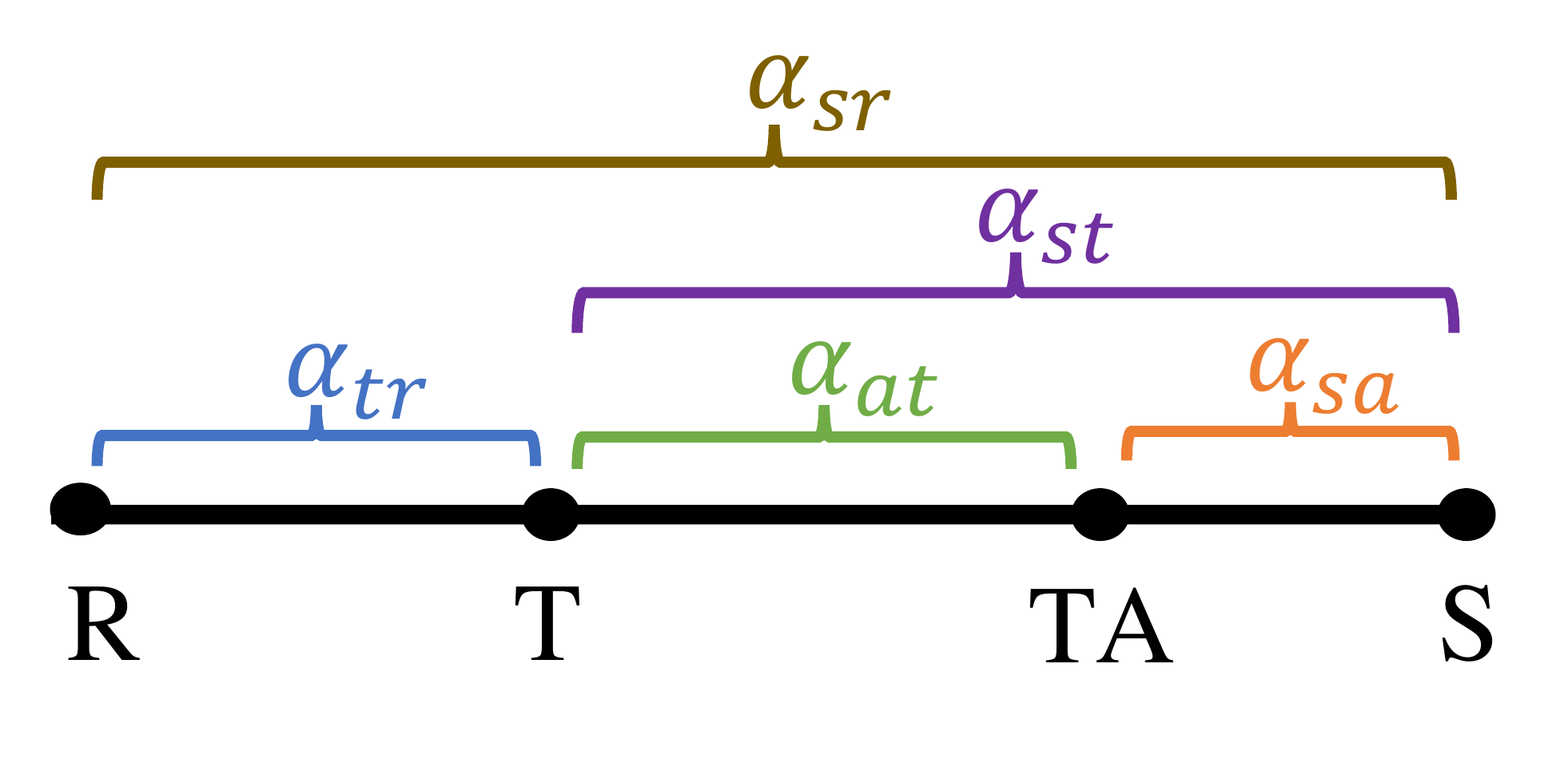} & 
\includegraphics[width=0.65\linewidth]{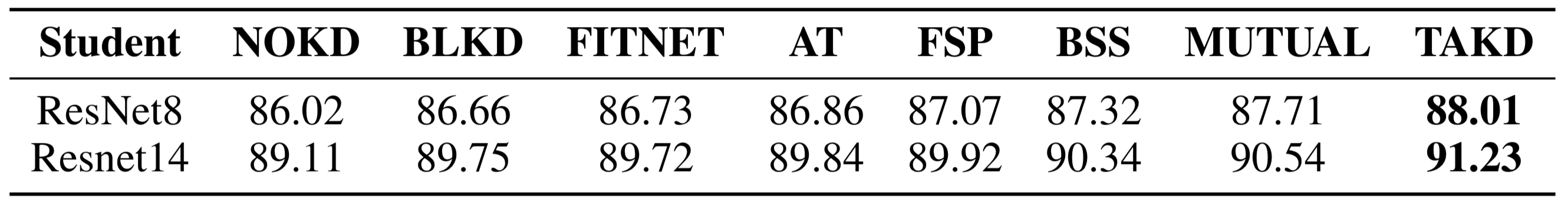}
\end{tabular}
\caption{Left) rate of learning between different targets (longer distance means lower $\alpha_{\cdot \cdot}$);
Right) Table for Comparison of TAKD with distillation alternatives on ResNet8 and ResNet14 as student and ResNet26  teacher} 
\label{fig:gaps}
\end{figure*}

\begin{figure*}[t!]
\centering
\resizebox{0.75\linewidth}{!}{

\begin{tabular}{ccc}
 \includegraphics[width=0.26\linewidth]{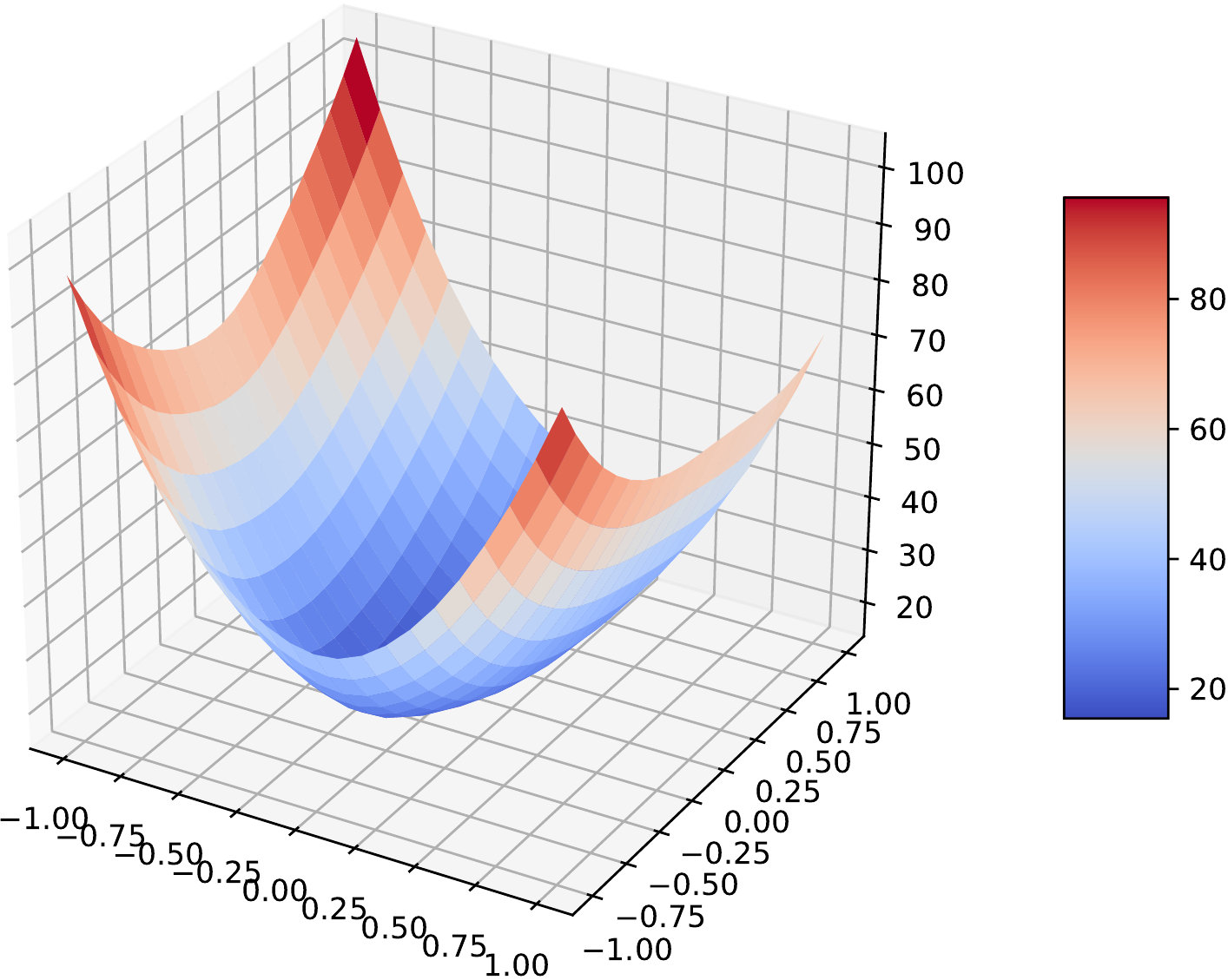} &
 \includegraphics[width=0.26\linewidth]{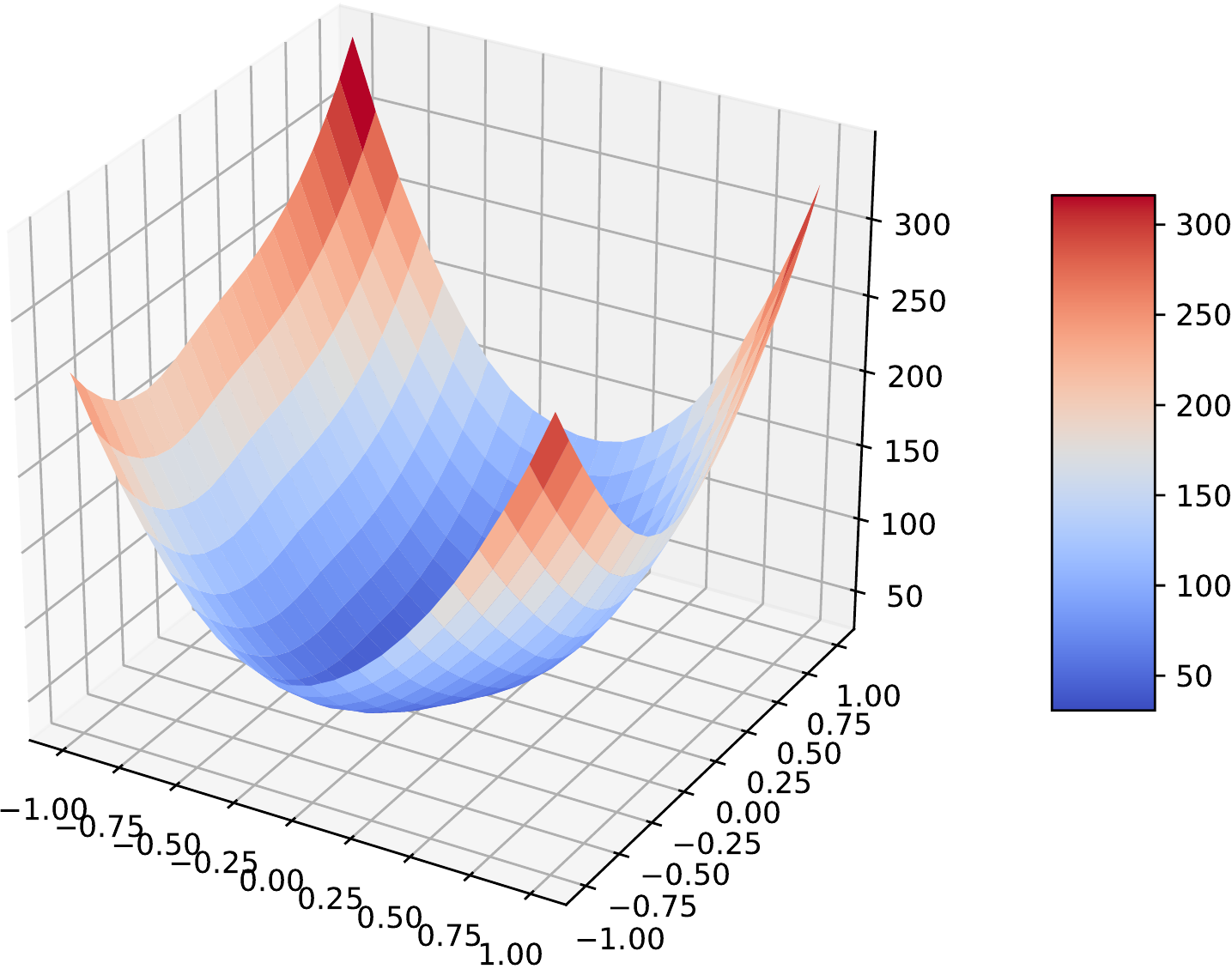} &
 \includegraphics[width=0.26\linewidth]{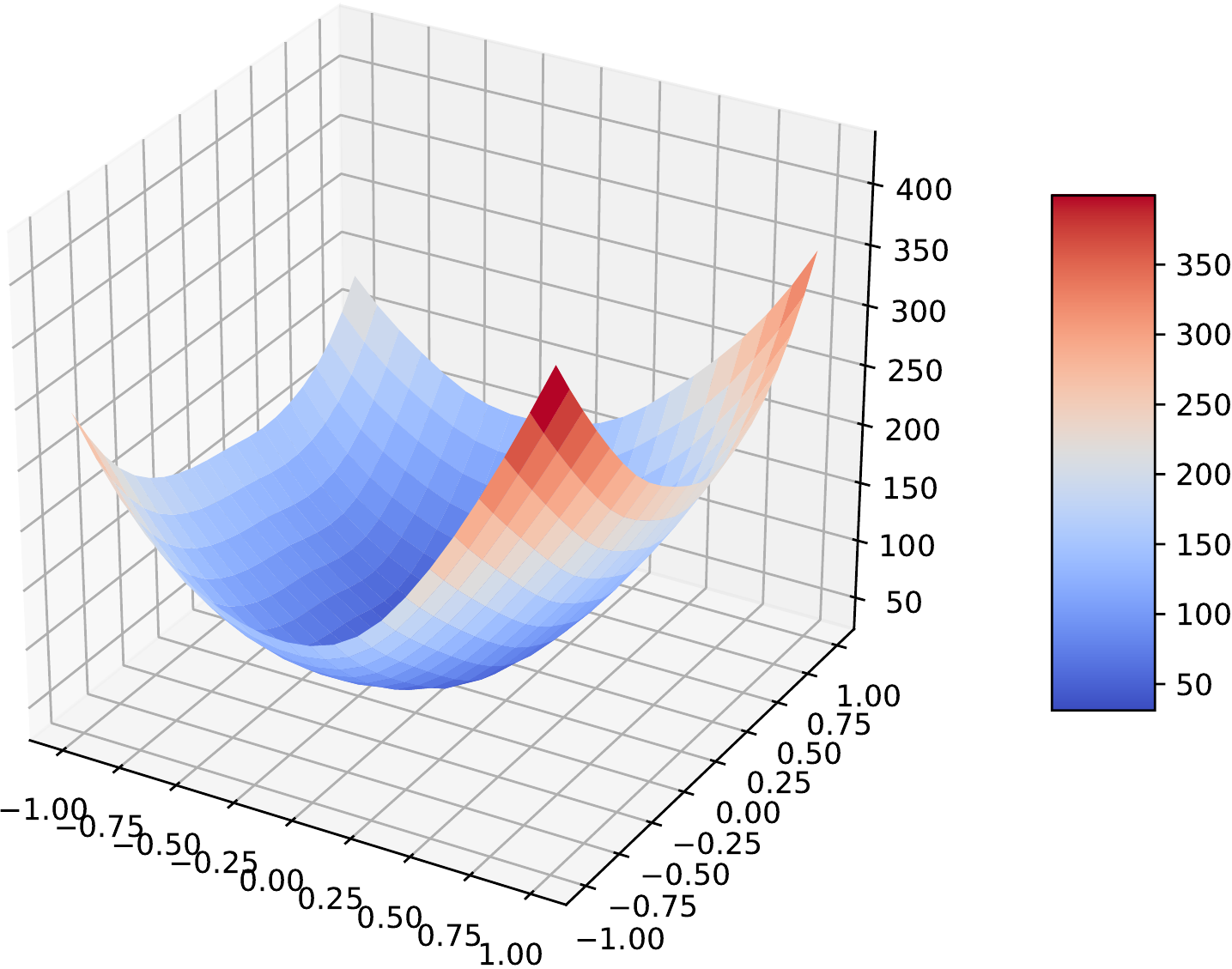} \\
 (a) NOKD & (b) BLKD ($10 \rightarrow 2$) & (c) TAKD ($10\rightarrow 4 \rightarrow 2$) \\
 \includegraphics[width=0.26\linewidth]{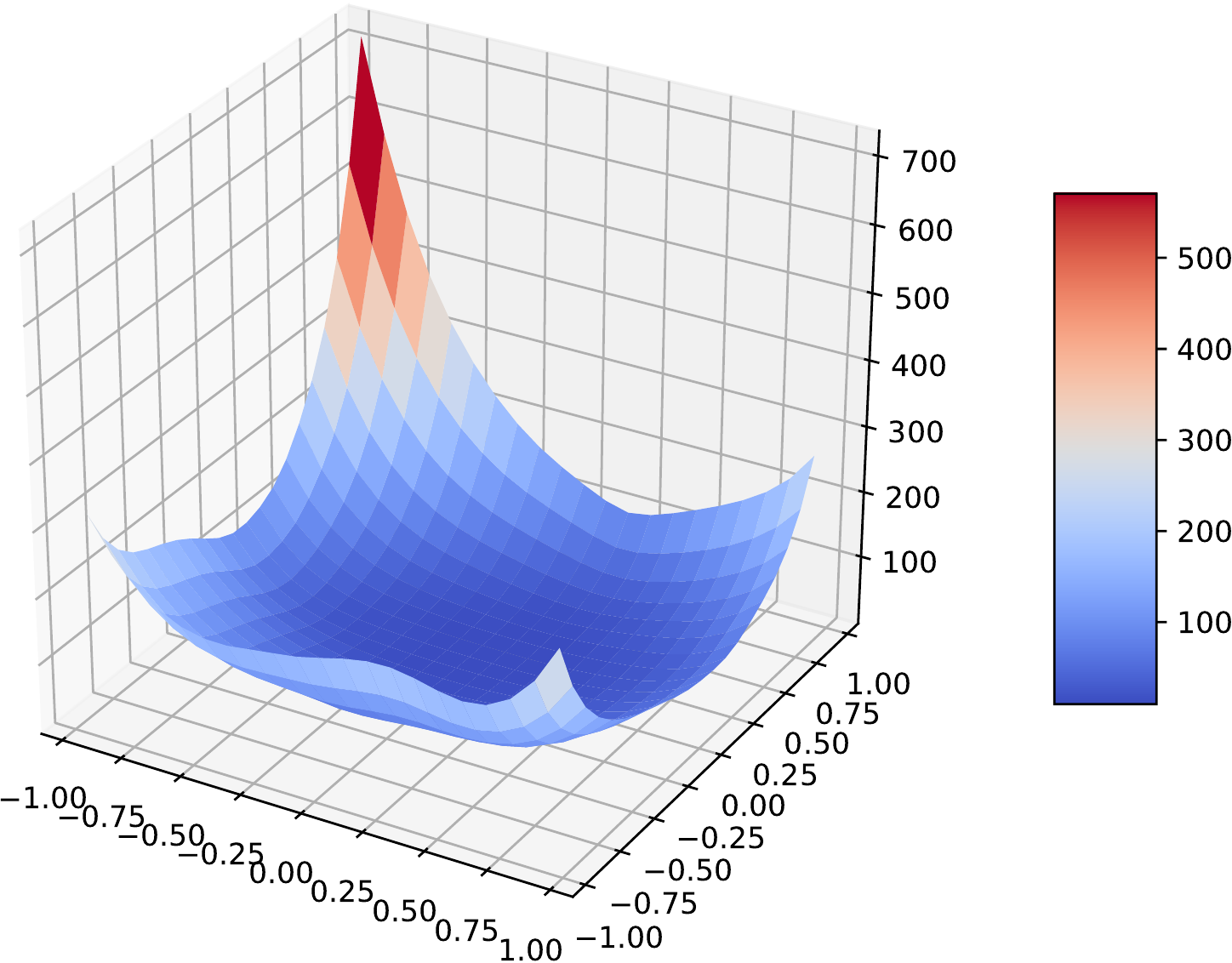} &
 \includegraphics[width=0.26\linewidth]{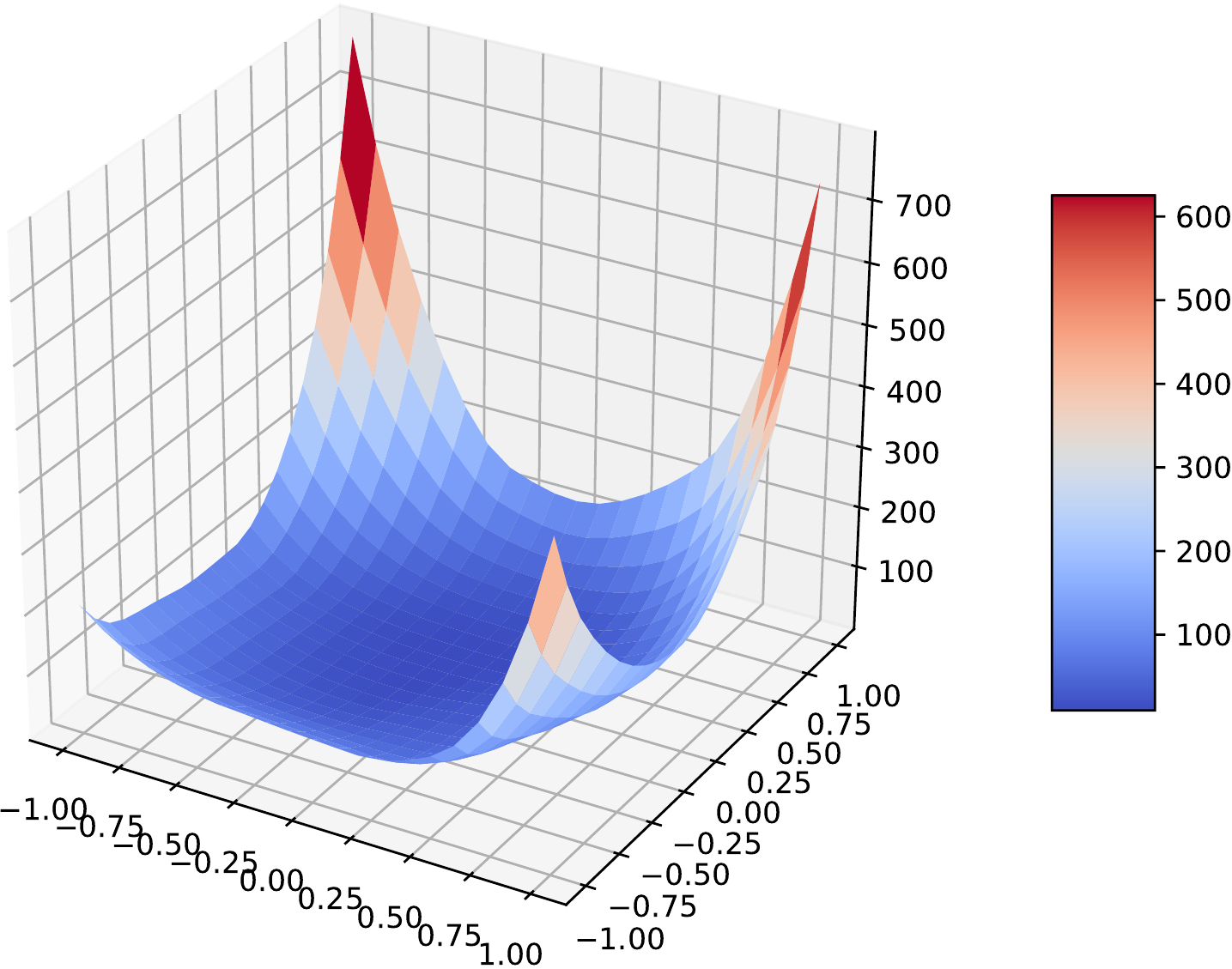} &
 \includegraphics[width=0.26\linewidth]{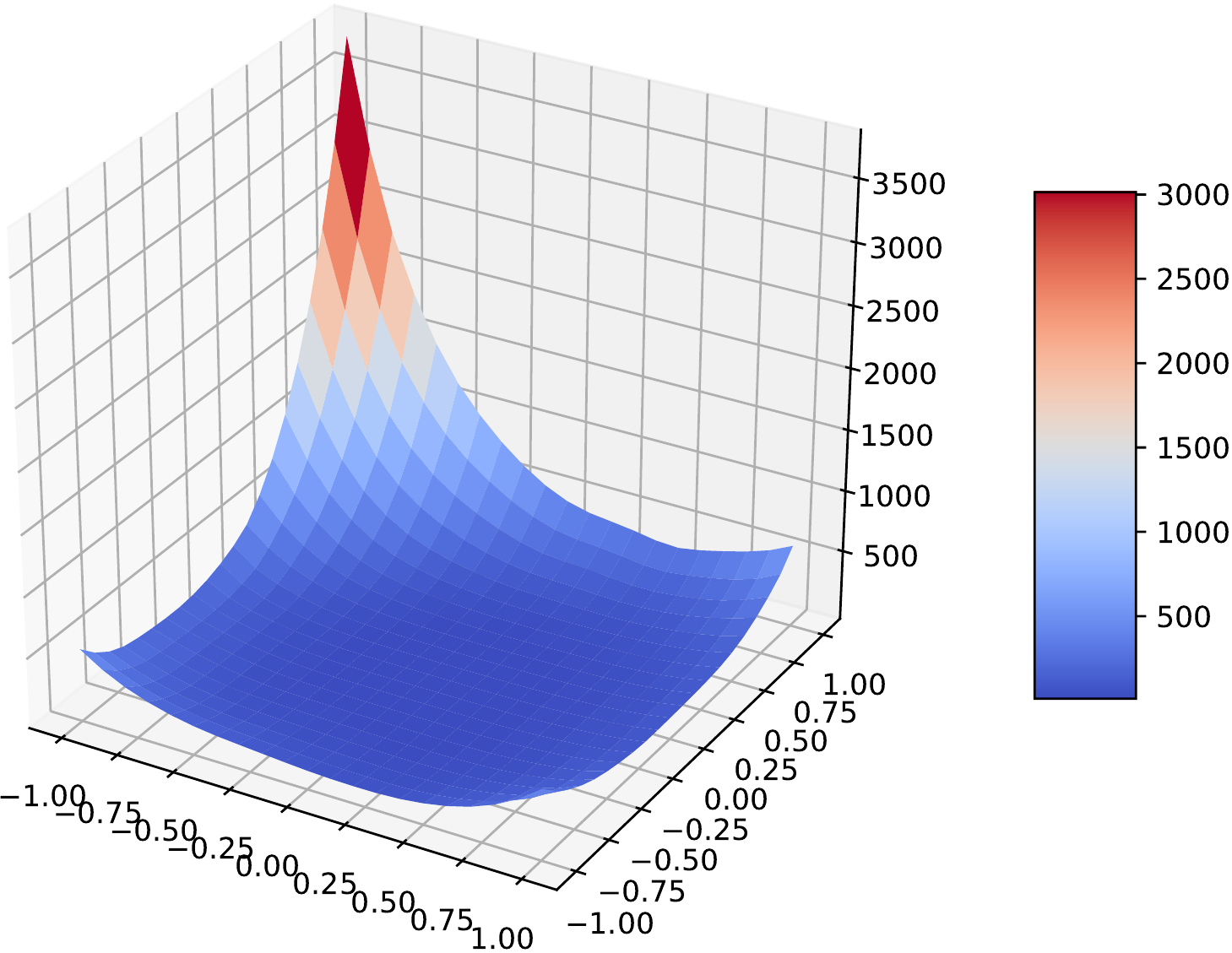} \\
  (d) NOKD & (e) BLKD ($110 \rightarrow 8$)  & (f) TAKD ($110 \rightarrow  20 \rightarrow 8$)
\end{tabular}
}
\caption{Loss landscape around local minima. Top) plain CNN for student of size 2. Bottom: ResNet for student of size 8.
}
\label{fig:loss}
\end{figure*}

\begin{align}
 O \left(
    \frac{|\mathcal{F}_t|_C}{n^{\alpha_{tr}}}
    + \frac{|\mathcal{F}_a|_C}{n^{\alpha_{at}}}
    + \frac{|\mathcal{F}_s|_C}{n^{\alpha_{sa}}}
    \right)  
    +\epsilon_{tr}
    +\epsilon_{at}
    +\epsilon_{sa} \label{eq:takd-1} & 
    \\ 
     \le
      O \left(
    \frac{|\mathcal{F}_t|_C}{n^{\alpha_{tr}}}
    + \frac{|\mathcal{F}_s|_C}{n^{\alpha_{st}}}
    \right) + \epsilon_{tr} + \epsilon_{st} &
    \label{eq:takd-2}
    \\
    \le 
    O \left(
    \frac{|\mathcal{F}_s|_C}{n^{\alpha_{sr}}} \right) + \epsilon_{sr} &.
    \label{eq:takd-3}
\end{align}
We now discuss how the first inequality  (eq. \eqref{eq:takd-1} $\le$ eq. \eqref{eq:takd-2}) holds which entails TAKD outperforms BLKD. To do so, first note that 
$\alpha_{st} \le \alpha_{sa}$ and $\alpha_{st} \le \alpha_{at}$ (the larger the gap means the lower rate of learning or smaller $\alpha_{\cdot \cdot}$). Figure~
\ref{fig:gaps}-left shows their differences. Student learning directly from teacher is certainly more difficult than either student learning from TA or TA learning from teacher. Therefore, asymptotically speaking, 
$ O \left( \frac{|\mathcal{F}_a|_C}{n^{\alpha_{at}}}
    + \frac{|\mathcal{F}_s|_C}{n^{\alpha_{sa}}}
    \right) \le
    O \left(
    \frac{|\mathcal{F}_s|_C}{n^{\alpha_{st}}}
    \right)
$ which in turn leads to
$O \left(
    \frac{|\mathcal{F}_t|_C}{n^{\alpha_{tr}}}
    + \frac{|\mathcal{F}_a|_C}{n^{\alpha_{at}}}
    + \frac{|\mathcal{F}_s|_C}{n^{\alpha_{sa}}}
    \right) \le O \left(
    \frac{|\mathcal{F}_t|_C}{n^{\alpha_{tr}}}
    + \frac{|\mathcal{F}_s|_C}{n^{\alpha_{st}}}
    \right) $.
    Moreover, according to assumption of~\citet{kd} we know $\epsilon_{at}
    +\epsilon_{sa} \le \epsilon_{st}$. These two together establish eq.~\eqref{eq:takd-1} $\le$ eq.~\eqref{eq:takd-2}, which means that the upper bound of error in TAKD is smaller than its upper bound in BLKD.
    
    Similarly, for the second inequality (eq.~\eqref{eq:takd-2} $\le$ eq.~\eqref{eq:takd-3}) one can use $\alpha_{sr} \le \alpha_{st}$ and  $\alpha_{sr} \le \alpha_{tr}$ and $\epsilon_{tr} + \epsilon_{st} \le \epsilon_{sr}$. Note that, these are asymptotic equations and hold when $n \to \infty$. In the finite sample regime, when $|\mathcal{F}_t|_C$ is very large, then the inequality 
    eq.~\eqref{eq:takd-2} $\le$ eq.~\eqref{eq:takd-3}
    may not be valid and BLKD fails. 
  Another failure case (in the finite sample regime) for BLKD happens when the student and teacher differ greatly in the capacity  (i.e. $\alpha_{st}$ is very small and close to $\alpha_{sr}$). In this case, the error due to transfer from real to teacher outweigh \eqref{eq:takd-2} in comparison to \eqref{eq:takd-3} and the inequality becomes invalid. In this case TAKD turns to be the key. By injecting a TA between student and teacher we break the very small $\alpha_{st}$ to two larger components $\alpha_{sa}$ and $\alpha_{at}$ which makes the second inequality (eq.~\eqref{eq:takd-1} $\le$
  eq.~\eqref{eq:takd-2}) a game changer for improving knowledge distillation.

\subsection{Empirical Analysis}
Whether or not a smooth (or sharp) loss landscape is related to the generalization error, is under an active debate in the general machine learning community~\citep{li2018visualizing}. 
However, for the case of knowledge distillation it seems to have connections to better accuracy. It's believed  that softened targets provide information on the similarity between output categories~\citep{kd}. 
\citet{furlanello2018born} connected the knowledge distillation to a weighted/smoothed loss over classification labels. Importantly,~\citet{deep-mutual} used posterior entropy and its flatness to make sense of the success of knowledge distillation. 
Supported by these prior works we propose to analyze the KD methods through loss landscape. In Figure~\ref{fig:loss}, using a recent state of the art landscape visualization technique~\citep{li2018visualizing} the loss surface of plain CNN on CIFAR-100 is plotted for student in three modes: (1) no knowledge distillation (NOKD), (2) baseline knowledge distillation (BLKD), (3) the proposed method (TAKD). It's seen that our network has a flatter surface around the local minima. This is related to robustness against noisy inputs which leads to better generalization.


\section{Summary}
\label{sec:conclusion}
We studied an under-explored yet important property in  Knowledge Distillation of neural networks. We showed that the gap between student and teacher networks is a key to the efficacy of knowledge distillation and the student network performance may decrease when the gap is larger.
 We proposed a framework based on Teacher Assistant knowledge Distillation to remedy this situation. We demonstrated the effectiveness of our approach in various scenarios and studied its properties both empirically and theoretically.
Designing a fully data-driven automated TA selection is an interesting venue for future work. We also would like to make a call for research on deriving tighter theoretical bounds and rigorous analysis for knowledge distillation.

\section*{Acknowledgement}
Authors SIM and HGH were supported in part  through  grant CNS-1750679 from the United States National Science Foundation. The authors would like to thank Luke Metz, Rohan Anil, Sepehr Sameni, Hooman Shahrokhi, Janardhan Rao Doppa, and Hung Bui for their feedback.

\bibliographystyle{aaai}
\bibliography{main}

\begin{thebibliography}{}

\bibitem[\protect\citeauthoryear{Anil \bgroup et al\mbox.\egroup
  }{2018}]{codistillation}
Anil, R.; Pereyra, G.; Passos, A.; Orm{\'{a}}ndi, R.; Dahl, G.; and Hinton, G.
\newblock 2018.
\newblock Large scale distributed neural network training through online
  distillation.
\newblock {\em CoRR} abs/1804.03235.

\bibitem[\protect\citeauthoryear{Ba and Caruana}{2014}]{ba2014deep}
Ba, J., and Caruana, R.
\newblock 2014.
\newblock Do deep nets really need to be deep?
\newblock In {\em NIPS},  2654--2662.

\bibitem[\protect\citeauthoryear{Bergstra \bgroup et al\mbox.\egroup
  }{2011}]{Bergstra2011AlgorithmsFH}
Bergstra, J.; Bardenet, R.; Bengio, Y.; and K{\'e}gl, B.
\newblock 2011.
\newblock Algorithms for hyper-parameter optimization.
\newblock In {\em NIPS}.

\bibitem[\protect\citeauthoryear{Bertsekas}{2005}]{bertsekas2005dynamic}
Bertsekas, D.
\newblock 2005.
\newblock {\em Dynamic programming and optimal control}, volume~1.
\newblock Athena scientific Belmont, MA.

\bibitem[\protect\citeauthoryear{Bucila, Caruana, and
  Niculescu-Mizil}{2006}]{bucilua2006model}
Bucila, C.; Caruana, R.; and Niculescu-Mizil, A.
\newblock 2006.
\newblock Model compression.
\newblock In {\em SIGKDD},  535--541.
\newblock ACM.

\bibitem[\protect\citeauthoryear{Chen \bgroup et al\mbox.\egroup
  }{2017}]{chen2017learning}
Chen, G.; Choi, W.; Yu, X.; Han, T.; and Chandraker, M.
\newblock 2017.
\newblock Learning efficient object detection models with knowledge
  distillation.
\newblock In {\em NIPS},  742--751.

\bibitem[\protect\citeauthoryear{Czarnecki \bgroup et al\mbox.\egroup
  }{2017}]{czarnecki2017sobolev}
Czarnecki, W.; Osindero, S.; Jaderberg, M.; Swirszcz, G.; and Pascanu, R.
\newblock 2017.
\newblock Sobolev training for neural networks.
\newblock In {\em NIPS},  4278--4287.

\bibitem[\protect\citeauthoryear{Devlin \bgroup et al\mbox.\egroup
  }{2018}]{Devlin2018BERTPO}
Devlin, J.; Chang, M.; Lee, K.; and Toutanova, K.
\newblock 2018.
\newblock Bert: Pre-training of deep bidirectional transformers for language
  understanding.
\newblock {\em CoRR} abs/1810.04805.

\bibitem[\protect\citeauthoryear{Furlanello \bgroup et al\mbox.\egroup
  }{2018}]{furlanello2018born}
Furlanello, T.; Lipton, Z.; Tschannen, M.; Itti, L.; and Anandkumar, A.
\newblock 2018.
\newblock Born again neural networks.
\newblock {\em arXiv preprint arXiv:1805.04770}.

\bibitem[\protect\citeauthoryear{Han \bgroup et al\mbox.\egroup
  }{2017}]{Han2017TheC2}
Han, K.~J.; Chandrashekaran, A.; Kim, J.; and Lane, I.
\newblock 2017.
\newblock The capio 2017 conversational speech recognition system.
\newblock {\em CoRR} abs/1801.00059.

\bibitem[\protect\citeauthoryear{Han, Mao, and
  Dally}{2016}]{han2015compression}
Han, S.; Mao, H.; and Dally, W.~J.
\newblock 2016.
\newblock Deep compression: Compressing deep neural network with pruning,
  trained quantization and huffman coding.
\newblock In {\em 4th International Conference on Learning Representations,
  {ICLR} 2016}.

\bibitem[\protect\citeauthoryear{He \bgroup et al\mbox.\egroup
  }{2016}]{he2016deep}
He, K.; Zhang, X.; Ren, S.; and Sun, J.
\newblock 2016.
\newblock Deep residual learning for image recognition.
\newblock In {\em CVPR},  770--778.

\bibitem[\protect\citeauthoryear{Heo \bgroup et al\mbox.\egroup
  }{2018}]{heo2018improving}
Heo, B.; Lee, M.; Yun, S.; and Choi, J.
\newblock 2018.
\newblock Improving knowledge distillation with supporting adversarial samples.
\newblock {\em arXiv preprint arXiv:1805.05532}.

\bibitem[\protect\citeauthoryear{Hinton, Vinyals, and Dean}{2015}]{kd}
Hinton, G.; Vinyals, O.; and Dean, J.
\newblock 2015.
\newblock Distilling the knowledge in a neural network.
\newblock In {\em NIPS Deep Learning and Representation Learning Workshop}.

\bibitem[\protect\citeauthoryear{Hu, Shen, and
  Sun}{2018}]{Hu2018SqueezeandExcitationN}
Hu, J.; Shen, L.; and Sun, G.
\newblock 2018.
\newblock Squeeze-and-excitation networks.
\newblock In {\em CVPR}.

\bibitem[\protect\citeauthoryear{Huang \bgroup et al\mbox.\egroup
  }{2017}]{Huang2017DenselyCC}
Huang, G.; Liu, Z.; Maaten, L.; and Weinberger, K.
\newblock 2017.
\newblock Densely connected convolutional networks.
\newblock {\em CVPR}  2261--2269.

\bibitem[\protect\citeauthoryear{Li \bgroup et al\mbox.\egroup
  }{2016}]{Li2016PruningFF}
Li, H.; Kadav, A.; Durdanovic, I.; Samet, H.; and Graf, H.~P.
\newblock 2016.
\newblock Pruning filters for efficient convnets.
\newblock {\em CoRR} abs/1608.08710.

\bibitem[\protect\citeauthoryear{Li \bgroup et al\mbox.\egroup
  }{2018}]{li2018visualizing}
Li, H.; Xu, Z.; Taylor, G.; Studer, C.; and Goldstein, T.
\newblock 2018.
\newblock Visualizing the loss landscape of neural nets.
\newblock In {\em NIPS},  6391--6401.

\bibitem[\protect\citeauthoryear{Lopez-Paz \bgroup et al\mbox.\egroup
  }{2015}]{lopez2015unifying}
Lopez-Paz, D.; Bottou, L.; Sch{\"o}lkopf, B.; and Vapnik, V.
\newblock 2015.
\newblock Unifying distillation and privileged information.
\newblock {\em arXiv preprint arXiv:1511.03643}.

\bibitem[\protect\citeauthoryear{Microsoft-Research}{2018}]{MSNNI}
Microsoft-Research.
\newblock 2018.
\newblock Neural network intelligence toolkit.

\bibitem[\protect\citeauthoryear{Paszke \bgroup et al\mbox.\egroup
  }{2017}]{paszke2017automatic}
Paszke, A.; Gross, S.; Chintala, S.; and et. al.
\newblock 2017.
\newblock Automatic differentiation in pytorch.
\newblock In {\em NIPS Autodiff Workshop}.

\bibitem[\protect\citeauthoryear{Polino, Pascanu, and
  Alistarh}{2018}]{polino2018model}
Polino, A.; Pascanu, R.; and Alistarh, D.
\newblock 2018.
\newblock Model compression via distillation and quantization.
\newblock In {\em ICLR}.

\bibitem[\protect\citeauthoryear{Romero \bgroup et al\mbox.\egroup
  }{2014}]{romero2014fitnets}
Romero, A.; Ballas, N.; Kahou, S.; Chassang, A.; Gatta, C.; and Bengio, Y.
\newblock 2014.
\newblock Fitnets: Hints for thin deep nets.
\newblock {\em arXiv preprint arXiv:1412.6550}.

\bibitem[\protect\citeauthoryear{Sau and Balasubramanian}{2016}]{sau2016deep}
Sau, B., and Balasubramanian, V.
\newblock 2016.
\newblock Deep model compression: Distilling knowledge from noisy teachers.
\newblock {\em arXiv preprint arXiv:1610.09650}.

\bibitem[\protect\citeauthoryear{Schmitt \bgroup et al\mbox.\egroup
  }{2018}]{kickstart-rl}
Schmitt, S.; Hudson, J.; Zidek, A.; and et. al.
\newblock 2018.
\newblock Kickstarting deep reinforcement learning.
\newblock {\em CoRR} abs/1803.03835.

\bibitem[\protect\citeauthoryear{Tai \bgroup et al\mbox.\egroup
  }{2015}]{Tai2015ConvolutionalNN}
Tai, C.; Xiao, T.; Wang, X.; and Weinan, E.
\newblock 2015.
\newblock Convolutional neural networks with low-rank regularization.
\newblock {\em CoRR} abs/1511.06067.

\bibitem[\protect\citeauthoryear{Tarvainen and
  Valpola}{2017}]{Tarvainen2017MeanTA}
Tarvainen, A., and Valpola, H.
\newblock 2017.
\newblock Mean teachers are better role models: Weight-averaged consistency
  targets improve semi-supervised deep learning results.
\newblock In {\em NIPS}.

\bibitem[\protect\citeauthoryear{Urban \bgroup et al\mbox.\egroup
  }{2017}]{urban2016deep}
Urban, G.; Geras, K.~J.; Kahou, S.~E.; and et. al.
\newblock 2017.
\newblock Do deep convolutional nets really need to be deep and convolutional?
\newblock In {\em 5th International Conference on Learning Representations,
  {ICLR} 2017, Conference Track Proceedings}.

\bibitem[\protect\citeauthoryear{Vapnik}{1998}]{vapnik1998statistical}
Vapnik, V.
\newblock 1998.
\newblock {\em Statistical learning theory. 1998}, volume~3.
\newblock Wiley, New York.

\bibitem[\protect\citeauthoryear{Wang \bgroup et al\mbox.\egroup
  }{2018a}]{wang-aaai19}
Wang, J.; Bao, W.; Sun, L.; Zhu, X.; Cao, B.; and Yu, P.~S.
\newblock 2018a.
\newblock Private model compression via knowledge distillation.
\newblock {\em CoRR} abs/1811.05072.

\bibitem[\protect\citeauthoryear{Wang \bgroup et al\mbox.\egroup
  }{2018b}]{wang2018kdgan}
Wang, X.; Zhang, R.; Sun, Y.; and Qi, J.
\newblock 2018b.
\newblock Kdgan: Knowledge distillation with generative adversarial networks.
\newblock In {\em NIPS},  783--794.

\bibitem[\protect\citeauthoryear{Wang \bgroup et al\mbox.\egroup
  }{2018c}]{Wang2018AdversarialLO}
Wang, Y.; Xu, C.; Xu, C.; and Tao, D.
\newblock 2018c.
\newblock Adversarial learning of portable student networks.
\newblock In {\em AAAI}.

\bibitem[\protect\citeauthoryear{Xu, Hsu, and Huang}{2018}]{xu2018training}
Xu, Z.; Hsu, Y.; and Huang, J.
\newblock 2018.
\newblock Training shallow and thin networks for acceleration via knowledge
  distillation with conditional adversarial networks.
\newblock In {\em 6th International Conference on Learning Representations,
  {ICLR} 2018}.

\bibitem[\protect\citeauthoryear{Yim \bgroup et al\mbox.\egroup
  }{2017}]{kd-gift}
Yim, J.; Joo, D.; Bae, J.; and Kim, J.
\newblock 2017.
\newblock A gift from knowledge distillation: Fast optimization, network
  minimization and transfer learning.
\newblock In {\em CVPR},  7130--7138.

\bibitem[\protect\citeauthoryear{You \bgroup et al\mbox.\egroup
  }{2017}]{you2017learning}
You, S.; Xu, C.; Xu, C.; and Tao, D.
\newblock 2017.
\newblock Learning from multiple teacher networks.
\newblock In {\em SIGKDD},  1285--1294.
\newblock ACM.

\bibitem[\protect\citeauthoryear{Yu \bgroup et al\mbox.\egroup
  }{2017}]{Yu_2017_ICCV}
Yu, R.; Li, A.; Morariu, V.~I.; and Davis, L.
\newblock 2017.
\newblock Visual relationship detection with internal and external linguistic
  knowledge distillation.
\newblock In {\em ICCV}.

\bibitem[\protect\citeauthoryear{Yu \bgroup et al\mbox.\egroup
  }{2018}]{Yu_2018_CVPR}
Yu, R.; Li, A.; Chen, C.; Lai, J.; Morariu, V.; Han, X.; Gao, M.; Lin, C.; and
  Davis, L.
\newblock 2018.
\newblock Nisp: Pruning networks using neuron importance score propagation.
\newblock In {\em CVPR}.

\bibitem[\protect\citeauthoryear{Z and K}{2016}]{zagoruyko2016paying}
Z, S., and K, N.
\newblock 2016.
\newblock Paying more attention to attention: Improving the performance of
  convolutional neural networks via attention transfer.
\newblock {\em arXiv preprint arXiv:1612.03928}.

\bibitem[\protect\citeauthoryear{Zhang \bgroup et al\mbox.\egroup
  }{2017}]{deep-mutual}
Zhang, Y.; Xiang, T.; Hospedales, T.; and Lu, H.
\newblock 2017.
\newblock Deep mutual learning.
\newblock {\em CoRR} abs/1706.00384.

\end{thebibliography}


\clearpage
\newpage
\appendix
\section{Network Architectures}
\label{app-networks}
In this section we explain the exact architecture of the models used in experiments. In order to have a concise representation, we use the following abbreviations for different layer types: 
\begin{itemize}
\itemsep 0pt
\item CB means a convolutional layer followed by a batch normalization.
\item MP repreesnts a maxpooling layer. 
\item FC stands for a fully connected layer. 
\end{itemize}
All the convolutional layers use $3\times3$ filters and maxpooling layers have stride of $2$ and kernel size of $3$. Finally, the number after each layer type, is the number of output channels if the layer is convolutional or output units if it is fully connected. For example, CB32 represents a convolutional layer with $32$ output channels followed by a batch normalization layer.
Networks used in CIFAR-10 and  CIFAR-100 experiments are described in Table \ref{table:c10models} and Table \ref{table:c100models}, respectively.

\begin{table}[h!]
\caption{Plain CNN architecture in CIFAR-10 experiments}
\label{table:c10models}
\begin{tabular}{|c|l|}
\hline
\# Conv. Layers & Architecture \\ \hline
2 & CB16, MP, CB16, MP, Ã¢ÂÂFC10Ã¢ÂÂ \\ \hline
4 & \begin{tabular}[c]{@{}l@{}}CB16, CB16, MP, CB32, CB32,\\  MP, FC10\end{tabular} \\ \hline
6 & \begin{tabular}[c]{@{}l@{}}CB16, CB16, MP, CB32, CB32,\\  MP, CB64, CB64, MP, FC10\end{tabular} \\ \hline
8 & \begin{tabular}[c]{@{}l@{}}CB16, CB16, MP, CB32, CB32,\\ MP, CB64, CB64, MP, CB128,\\ CB128, MP, FC64, FC10\end{tabular} \\ \hline
10 & \begin{tabular}[c]{@{}l@{}}CB32, CB32, MP, CB64, CB64,\\  MP, CB128, CB128, MP, CB256,\\  CB256, CB256, CB256, MP,\\ FC128, FC10\end{tabular} \\ \hline
\end{tabular}
\end{table}


\begin{table}[h!]
\caption{Plain CNN architecture in CIFAR-100 experiments}
\label{table:c100models}
\begin{tabular}{|c|l|}
\hline
\# Conv. Layers & Architecture \\ \hline
2 & CB32, MP, CB32, MP, FC100 \\ \hline
4 & \begin{tabular}[c]{@{}l@{}}CB32, CB32, MP, CB64, CB64,\\ MP, FC100\end{tabular} \\ \hline
6 & \begin{tabular}[c]{@{}l@{}}CB32,C B32, MP, CB64, CB64,\\ MP, CB128, CB128, FC100\end{tabular} \\ \hline
8 & \begin{tabular}[c]{@{}l@{}}CB32, CB32, MP, CB64, CB64,\\ MP, CB128,  CB128, MP, CB256,\\ CB256, MP, FC64, FC100
\end{tabular} \\ \hline
10 & \begin{tabular}[c]{@{}l@{}}CB32, CB32, MP, CB64, CB64, MP,\\ CB128, CB128, MP, CB256, CB256,\\ CB256, CB256, MP, FC512, FC100\end{tabular} \\ \hline
\end{tabular}
\end{table}

\begin{figure}[t!]
\centering
\begin{tabular}{cc}
\hspace{-3mm}
\includegraphics[width=0.51\linewidth]{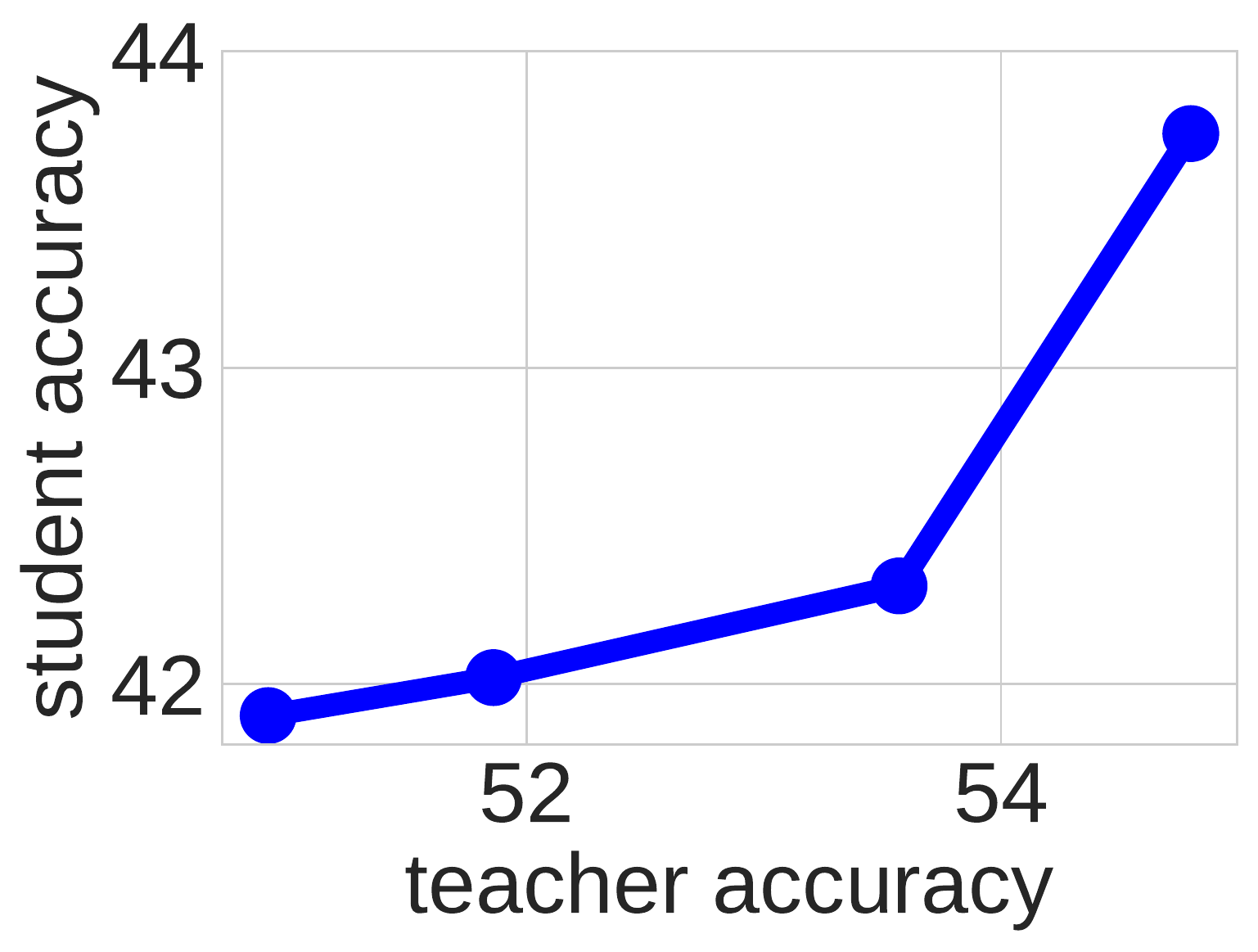} & \hspace{-4mm}
\includegraphics[width=0.49\linewidth]{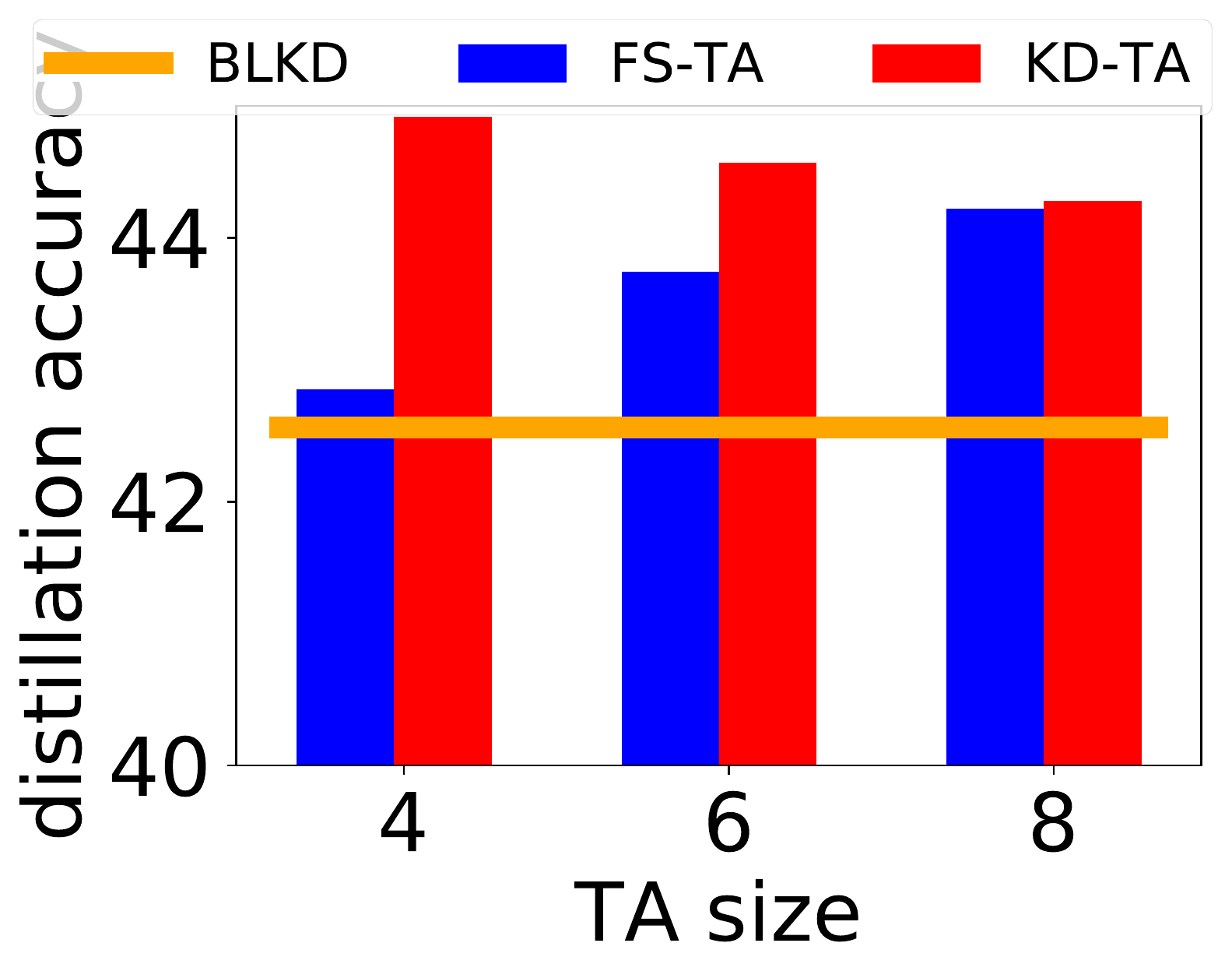} \\
(a) & (b)
\end{tabular}
\caption{ (a) Distillation performance vs. accuracy of the teacher (all teachers have the same number of convolutional layers); (b) Comparison of distillation performance via distilled TA and from-scratch TA.}
\label{fig:why-distil-ta}
\end{figure}

\section{Why Using a Distilled TA?}
\label{subsec-why-distil-ta}
One might wonder why the intermediate TA needs to be trained via distillation while training from scratch is an option too. We show in this section that the importance of the TA is not only related to its size, but also to the way it is trained and what accuracy it achieves. For brevity, we focus on plain CNN architecture and CIFAR-100 dataset.

We fix the network to have $6$ convolutional layers  and train it (without knowledge distillation) by varying the number of epochs within the set $\{5, 10, 15, 100\}$. This leads to 4 networks with increased performance that we will use as teachers in knowledge distillation while the student size is 2.
Figure~\ref{fig:why-distil-ta}-a shows the results.
It's seen that a better network (in terms of accuracy) is a better teacher. Moreover, we know that knowledge distillation usually leads to better networks compared to being trained from scratch and also justified theoretically. Therefore, it's justifiable to use distillation to train the intermediary networks (i.e., TAs).

Moreover, we perform an empirical investigation to validate that these distilled TAs indeed lead to better students. Figure~\ref{fig:why-distil-ta}-b shows the distillation performance with distilled TA (KD-TA) and with from scratch TAs (FS-TA) which are trained only from the data. It is seen that distilled TAs are more successful in training the students. Here the student size is 2 and the teacher size is 10.
Also, note that when TA size is larger (8 e.g.) the difference between FS-Ta and KD-TA is small. This might be due to still large size gap between TA (8) and student (2) that makes the distillation procedure less effective.


\section{The Best TA Sequence}
One of the trending challenges for machine learning is moving towards a real \emph{ automatic machine learning} (AutoML) where the algorithms run with minimum human supervision. Our proposed framework, Teacher Assistant based Knowledge Distillation (TAKD), can also be adapted and deployed to that by automatically finding the best TA sequence. Given a fixed expert neural network (T) and a small student neural network (S), when there is no resource or time constraint the best sequence to use T in training S is to sequentially use every TA possible in between in a decreasing order of capacity. To see why let's proceed with a more formal language. The aim of this part is mostly to encourage systematic study of knowledge distillation and to lay some language for further advances.

\begin{principle}[Knowledge Distillation]
\label{princ-kd}
Knowledge distillation improves the accuracy of the student network compared to student learning alone using classification loss only.
\end{principle}

\begin{principle}[Knowledge Distillation Performance]
\label{princ-distil}
Given teacher networks of the same capacity/complexity, the one with higher accuracy is a better teacher for a student in knowledge distillation framework.
\end{principle}

\begin{principle}[Teacher Assistant based Knowledge Distillation]
\label{princ-takd}
Introducing a teacher assistant between student and teacher improves over the baseline knowledge distillation. 
\end{principle}

The reason that we have not put the above statements as theorems or lemmas but principles is that, even though, there are a few papers (such as the current work or~\citep{lopez2015unifying}) which tried to build the initial steps towards a theoretical understanding of knowledge distillation, most of the works in this area are only empirically validated. 
For example, in our work, Principle~\ref{princ-distil} is verified in the previous section that among the teachers of the same size the one with better performance is seen to be a better trainer for the student. In this section, we take these principles for granted and don't argue their correctness and build the rest on their top.
A rigorous mathematical understanding of the mechanics of knowledge distillation and yet more extensive empirical validation remain as  future work.

\begin{lemma}
\label{lem:opt-seq}
The optimal sequence for TAKD with multiple TAs consists of all available intermediate TA Networks.
\end{lemma}
\begin{proof}
This lemma can be verified by a simple proof by contradiction. Assume you order available networks by their capacity (\emph{aka} flexibility or size) as $\mathbb{Q} = (q_0, q_1, q_2, \ldots, q_n)$ where $T=q_0$ is teacher size and $S=q_n$ is the student size. If the optimal sequence among the possible $2^{n-1}$ ones, $\mathbb{C} = (c_0=T, c_1, c_2, \ldots, \ldots, c_{m-1}, c_m=S)$ is not equal to $\mathbb{Q}$, then get the first $i$ such that $q_i > c_i$. This means network $q_i$ is missing from $\mathbb{C}$. If you add $q_i$ between $c_{i-1}$ and $c_{i+1}$ as a TA according to principle~\ref{princ-takd} the network $c_{i+1}$ improves. 
Then, according to principle~\ref{princ-distil} a better $c_{i+1}$ leads to better $c_{i+2}$ and so on until $S$.
Not only $q_i$ should be added but according to principle~\ref{princ-kd} it's always better to train it using knowledge distillation.
In summary, by this replacement we will get a better sequence which contradicts $\mathbb{C}$ being optimal. Therefore, $\mathbb{Q}$ is an optimal distillation path.
\end{proof}
The above result is obtained as expected. But, now one could ask a more interesting question: \emph{What is the best sequence when the number of TAs is constrained?} It is of practical implication when there is a limit in resource or time. For example, what are the best two TAs to distill a vanilla CNN-10 to a CNN-2? What are the best length-3 TA path to transfer knowledge from ResNet-110 to ResNet-8? Given the exponential number of sequences it can be undesirable to do an exhaustive search as we did in subsection~\ref{subsec:multi-ta} and depicted in Figure~\ref{fig:multi-ta}.
Fortunately, we can show that this problem satisfies \emph{the principle of optimality} and an efficient \emph{dynamic programming} solution exists~\citep{bertsekas2005dynamic}.

\begin{lemma}
The problem of finding the best length-$k$ TA sequence to distill the knowledge from teacher to student has optimal substructure property~\citep{bertsekas2005dynamic}.
\end{lemma}
\begin{proof}
Assume that the optimal length-$k$ TA sequence from a network of size $T$ to a network of size $S$ is $\mathbb{Q} = (q_0, q_1, \ldots, q_{k-1}, q_k)$, where $q_0=T$ is the teacher size and $ q_k = S$ is the student size. We claim that the path $\mathbb{Q'} = (q_0, q_1, \ldots, q_{k-1})$ is the optimal sequence of length $k-1$ to distill a teacher of size $q_0$ to $q_{k-1}$, otherwise, if there is a better sequence $\mathbb{L}' = c_0, c_1, \ldots, c_{k-1}$ from $c_0 = q_0$ to $c_{k-1} = q_{k-1}$, then replace $\mathbb{Q}'$ by $\mathbb{L}'$ in $\mathbb{Q}$. According to principle~\ref{princ-distil} we should get a better $\mathbb{Q}$ which is a contradiction. Therefore, if $\mathbb{Q}$ is the optimal sequence of length $k$ for distilling into network $q_{k}$, then $\mathbb{Q}'$ must be the optimal length $k-1$ for network $q_{k-1}$.
\end{proof}
The above optimal substructure entails a dynamic programming solution for the problem.
Define $p_k(q_j)$ as the best path of length $k$ for distilling teacher network T to a student network of size $q_j$ with $k$ distillation steps via intermediate TAs. Let $w_k(q_j)$ be the associated optimal network of size $q_j$.
Then,
\begin{equation}
    w_k(q_j) = \arg\min_{w\in \mathbf W_{j,k}} \ell(w),
\end{equation}
where $\mathbf W_{j,k}=\{\text{BLKD}(w_{k-1}(q_i)\to q_j)~|~0<i<j\}$. BLKD$(w \to q)$ means performing a direct (baseline) knowledge distillation and returning the trained network. $\ell$ evaluates and return the loss function. The argmin operator then returns the most accurate network.
If $i^*$ is the minimizer then the path $p_k(q_j) = [p_{k-1}(q_{i^*}); q_j]$. One can easily see that the optimal path will be found performing $O(kn^2)$ knowledge distillation operations, thanks to the beauty of dynamic programming. This is an exponential speedup  compared to the exhaustive search among ${n-1} \choose{k-1}$ possible paths. See Algorithm~\ref{alg:path} for the procedure.

\begin{algorithm}[t]
   \caption{Optimal length $k$ distillation path from T to S}
   \label{alg:path}
\begin{algorithmic}
   \STATE {\bfseries Input:}  $(q_0=T, q_1, q_2, \ldots, q_n=S)$ in decreasing order of size, length $k$
  \STATE {\bfseries Output:}  Optimal distilled student $w_k(S)$ and optimal distillation path $p_k(S)$
   \STATE {\# Initialization}
   \FOR{$i=1$ {\bfseries to} $n$}
   \STATE $w_{1}(q_i) = \text{BLKD}(T \to q_i)$
   \ENDFOR
    \STATE { \# Main Loop}
    \FOR{$d=2$ {\bfseries to} $k$}
     \FOR{$j=1$ {\bfseries to} $n$}
     
     \STATE $i^*=\arg\min_{0<i<j} \ell(\text{BLKD}(w_{d-1}(q_i)\to q_j))$
     \STATE $w_d(q_j) = \text{BLKD}(w_{d-1}(q_{i^*})\to q_j)$
     \STATE 
     $p_d(q_j) = [p_{d-1}(q_{i^*}); q_j]$
    \ENDFOR
    \ENDFOR
    \STATE \textbf{return} $p_k(S)$;
\end{algorithmic}
\end{algorithm}

We can also validate the optimality substructure of the problem thorough our experiemnts, where, Figure~\ref{fig:multi-ta} shows all the paths for T=10 and S=2. As an instance, let's examine the best path with exactly 2 TAs in between. Among $\mathbb{Q}_1 = (10 \to 8 \to 6 \to 2)$, $\mathbb{Q}_2 = (10 \to 8 \to 4 \to 2)$, and $\mathbb{Q}_3 = (10 \to 6 \to 4 \to 2)$, $\mathbb{Q}_3$ is the best, so we expect $\mathbb{Q}_3' = (10 \to 6 \to 4)$ be the best path to $4$ with one intermediate TA, which is the case in Figure~\ref{fig:multi-ta}.



\end{document}